\PassOptionsToPackage{unicode}{hyperref}
\PassOptionsToPackage{hyphens}{url}
\PassOptionsToPackage{dvipsnames,svgnames,x11names}{xcolor}
\documentclass[12pt]{article}
\usepackage{tcolorbox}
\usepackage{enumitem}
\usepackage{amsmath,amssymb}
\usepackage{amsthm}
\usepackage{bm}
\usepackage{multirow}
\usepackage{iftex}
\ifPDFTeX
  \usepackage[T1]{fontenc}
  \usepackage[utf8]{inputenc}
  \usepackage{textcomp} 
\else 
  \usepackage{unicode-math}
  \defaultfontfeatures{Scale=MatchLowercase}
  \defaultfontfeatures[\rmfamily]{Ligatures=TeX,Scale=1}
\fi
\usepackage{lmodern}
\usepackage{mathtools}
\usepackage[colorlinks]{hyperref}
\hypersetup{
    colorlinks,
	linkcolor={magenta},
    filecolor={cyan},
	citecolor={blue},
}
\ifPDFTeX\else  
\fi
\IfFileExists{upquote.sty}{\usepackage{upquote}}{}
\IfFileExists{microtype.sty}{
  \usepackage[]{microtype}
  \UseMicrotypeSet[protrusion]{basicmath} 
}{}
\makeatletter
\@ifundefined{KOMAClassName}{
  \IfFileExists{parskip.sty}{%
    \usepackage{parskip}
  }{
    \setlength{\parindent}{0pt}
    \setlength{\parskip}{6pt plus 2pt minus 1pt}}
}{
  \KOMAoptions{parskip=half}}
\makeatother
\usepackage{xcolor}
\makeatletter
\ifx\paragraph\undefined\else
  \let\oldparagraph\paragraph
  \renewcommand{\paragraph}{
    \@ifstar
      \xxxParagraphStar
      \xxxParagraphNoStar
  }
  \newcommand{\xxxParagraphStar}[1]{\oldparagraph*{#1}\mbox{}}
  \newcommand{\xxxParagraphNoStar}[1]{\oldparagraph{#1}\mbox{}}
\fi
\ifx\subparagraph\undefined\else
  \let\oldsubparagraph\subparagraph
  \renewcommand{\subparagraph}{
    \@ifstar
      \xxxSubParagraphStar
      \xxxSubParagraphNoStar
  }
  \newcommand{\xxxSubParagraphStar}[1]{\oldsubparagraph*{#1}\mbox{}}
  \newcommand{\xxxSubParagraphNoStar}[1]{\oldsubparagraph{#1}\mbox{}}
\fi
\makeatother

\newtheorem{theorem}{Theorem}

\newtheorem{lemma}{Lemma}

\newtheorem{assumption}{Assumption}

\usepackage{longtable,booktabs,array}
\usepackage{calc} 
\usepackage{etoolbox}
\makeatletter
\patchcmd\longtable{\par}{\if@noskipsec\mbox{}\fi\par}{}{}
\makeatother
\IfFileExists{footnotehyper.sty}{\usepackage{footnotehyper}}{\usepackage{footnote}}
\makesavenoteenv{longtable}
\usepackage{graphicx}
\makeatletter
\def\maxwidth{\ifdim\Gin@nat@width>\linewidth\linewidth\else\Gin@nat@width\fi}
\def\maxheight{\ifdim\Gin@nat@height>\textheight\textheight\else\Gin@nat@height\fi}
\makeatother
\setkeys{Gin}{width=\maxwidth,height=\maxheight,keepaspectratio}
\makeatletter
\def\fps@figure{htbp}
\makeatother

\makeatletter
\@ifpackageloaded{caption}{}{\usepackage{caption}}
\AtBeginDocument{%
\ifdefined\contentsname
  \renewcommand*\contentsname{Table of contents}
\else
  \newcommand\contentsname{Table of contents}
\fi
\ifdefined\listfigurename
  \renewcommand*\listfigurename{List of Figures}
\else
  \newcommand\listfigurename{List of Figures}
\fi
\ifdefined\listtablename
  \renewcommand*\listtablename{List of Tables}
\else
  \newcommand\listtablename{List of Tables}
\fi
\ifdefined\figurename
  \renewcommand*\figurename{Figure}
\else
  \newcommand\figurename{Figure}
\fi
\ifdefined\tablename
  \renewcommand*\tablename{Table}
\else
  \newcommand\tablename{Table}
\fi
}
\@ifpackageloaded{float}{}{\usepackage{float}}
\floatstyle{ruled}
\@ifundefined{c@chapter}{\newfloat{codelisting}{h}{lop}}{\newfloat{codelisting}{h}{lop}[chapter]}
\floatname{codelisting}{Listing}

\makeatother
\makeatletter
\@ifpackageloaded{caption}{}{\usepackage{caption}}
\@ifpackageloaded{subcaption}{}{\usepackage{subcaption}}
\makeatother

\ifLuaTeX
  \usepackage{selnolig}  
\fi
\usepackage[]{natbib}
\usepackage{bookmark}

\IfFileExists{xurl.sty}{\usepackage{xurl}}{} 
\hypersetup{
  pdftitle={Title},
  pdfauthor={Author 1; Author 2},
  pdfkeywords={3 to 6 keywords, that do not appear in the title},
  colorlinks=true,
  linkcolor={blue},
  filecolor={Maroon},
  citecolor={Blue},
  urlcolor={Blue},
  pdfcreator={LaTeX via pandoc}}

\usepackage{adjustbox}
\usepackage[table]{xcolor}
\definecolor{myblue}{HTML}{67A9CF}
\definecolor{darkgreen}{rgb}{0.0, 0.3, 0.0}

\usepackage{xr}
\newcommand{\anon}{1}

\begin{document}

\def\spacingset#1{\renewcommand{\baselinestretch}%
{#1}\small\normalsize} \spacingset{1}


\if1\anon
{
  \title{\bf Detecting LLM-Generated Text with Performance Guarantees}
  \author{Hongyi Zhou\textsuperscript{1*}, Jin Zhu\thanks{The first two authors contributed equally to this paper and are listed in alphabetical order. \\ $^\dag$Co-corresponding authors.}\;\textsuperscript{2$\dag$}, Ying Yang\textsuperscript{1}, Chengchun Shi\textsuperscript{3$\dag$}\\ \smallskip \\
    {$^1$Tsinghua University, $^2$University of Birmingham} $\&$ \\
    {$^3$London School of Economics and Political Science}\\
    }
\date{}
  \maketitle
} \fi

\if0\anon
{
  \bigskip
  \bigskip
  \bigskip
  \begin{center}
    {\LARGE\bf {\bf Detecting LLM-Generated Text with Performance Guarantees}}
\end{center}
  \medskip
} \fi

\bigskip
\begin{abstract}
Large language models (LLMs) such as GPT, Claude, Gemini, and Grok have been deeply integrated into our daily life. They now support a wide range of tasks -- from dialogue and email drafting to assisting with teaching and coding, serving as search engines, and much more. However, their ability to produce highly human-like text raises serious concerns, including the spread of fake news, the generation of misleading governmental reports, and academic misconduct. To address this practical problem, we train a classifier to determine whether a piece of text is authored by an LLM or a human. Our detector is deployed on an online CPU-based platform  \url{https://huggingface.co/spaces/stats-powered-ai/StatDetectLLM}, and contains three novelties over existing detectors: (i) it does not rely on auxiliary information, such as watermarks or knowledge of the specific LLM used to generate the text; (ii) it more effectively distinguishes between human- and LLM-authored text; and (iii) it enables statistical inference, which is largely absent in the current literature. Empirically, our classifier achieves higher classification accuracy compared to existing detectors, while maintaining type-I error control, high statistical power, and computational efficiency. 
\end{abstract}

\noindent%
\vfill

\newpage
\spacingset{1.8} 

\section{Introduction}\label{sec-intro}
The past few years have witnessed the rapid development of general-purpose large language models (LLMs) such as GPT \citep{hurst2024gpt}, DeepSeek \citep{liu2024deepseek}, Claude \citep{anthropic2024claude}, Gemini \citep{comanici2025gemini}, Grok \citep{xAI2023Grok} and Qwen \citep{yang2025qwen3}. These models have demonstrated remarkable performance across a wide range of tasks, from conventional question answering, summarization, translation to  reasoning and code generation. They are now deeply integrated into various application domains, including finance, education, healthcare, software engineering and journalism \citep{arora2023promise, chan2023students, hou2024large,liu2025fin}. 

Due to their ability to generate highly coherent, human-like text, these LLMs also pose serious societal and ethical challenges related to authorship attribution, academic integrity, intellectual property, and the spread of misinformation. 
For instance, in academics, authors may present LLM-generated ideas or writing as their own without giving proper credit. Similarly, reviewers may rely on LLM-generated reviews without thoroughly reading the paper, resulting in low-quality and unconstructive reviews. A recent study reports that at least 15.8\% of reviews submitted to the 2024 ICLR conference – one of the most prestigious machine learning conferences – were AI-assisted \citep{latona2024ai}. Meanwhile, at least 13.5\% of PubMed-indexed papers published in 2024 had abstracts that were processed with LLMs \citep{kobak2025delving}. Across social media platforms, LLMs can readily amplify the spread of disinformation at scale \citep{weidinger2021ethical}. Finally, policymakers and governments increasingly highlight the need to safeguard generative AI technologies to ensure they remain responsible and trustworthy \citep{oecd2024governingAI}. Consequently, it has become an urgent priority to develop reliable algorithms capable of distinguishing between human- and LLM-authored text \citep{crothers2023machine, wu2025survey}. 

\subsection{Related works}\label{sec:related}
\begin{figure}[t]
    \centering
    \includegraphics[width=0.85\linewidth]{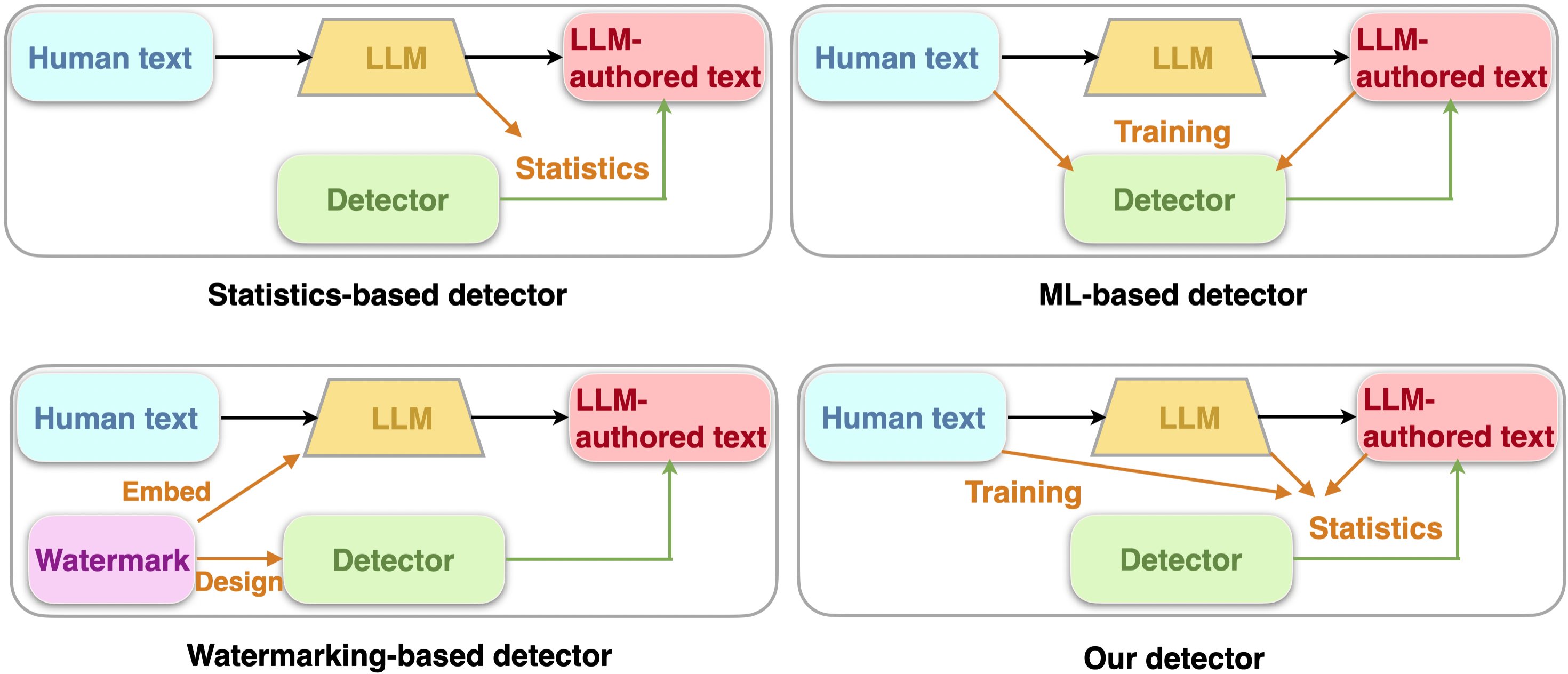}
    \caption{Visualizations of three types of existing detectors (statistics-based, ML-based, watermarking-based) along with our proposed detector.}\label{fig:relatedwork}
\end{figure}
Despite the recent emergence of LLMs, there is already a large body of research on detecting LLM-generated text. Broadly speaking, existing approaches fall into three categories: watermarking-based, machine learning (ML)-based, and statistics-based. We review each category below and provide a visualization of these approaches in Figure \ref{fig:relatedwork}. 

\begin{enumerate}[leftmargin=*]
\item \textbf{Watermarking-based} approaches insert subtle signals, known as watermarks, into a target LLM's output text (see the bottom left panel of Figure \ref{fig:relatedwork}). Text generated by this LLM can then be identified by testing the presence of these signals. An early example of this technique is Gumbel watermark \citep{aaronson2023watermarking}. More recently, a line of work in the statistics literature has framed watermark detection as a statistical hypothesis testing problem \citep[see, e.g.,][]{li2025robust,li2025statistical,xie2025debias}. Within this framework, the null hypothesis assumes that no watermark is present in the input text, and rejecting it provides statistical evidence that the text is generated by the target LLM. However, these approaches rely on knowledge of the specific hash function or random number generator (RNG) used during the target LLM's token generation, which is often not publicly available. 

\item \textbf{Statistics-based} approaches extract information from the target LLM to construct a test statistic, which is then used to distinguish between human- and machine-authored text (see the upper left panel of Figure \ref{fig:relatedwork}). As a simple example, suppose certain words are more likely to appear in LLM-generated text than in human-written text. Then the frequency of these words in the input text can serve as the test statistic: if the frequency is sufficiently high, we conclude that the text is  generated by the LLM. A variety of statistical measures have been proposed in the literature. Among these, the most commonly used are based on the logits of the target LLM’s next-token prediction distribution \citep{mitchell2023detectgpt, su2023detectllm, bao2024fastdetectgpt, hans2024spotting, zhou2025adadetectgpt}; see Section \ref{sec:preliminary} for further details of these logits-based methods. Other statistics include the input text’s N-gram distribution \citep{solaiman2019release, yang2024dnagpt}, its intrinsic dimensionality \citep{tulchinskii2023intrinsic},  the distribution of
absolute ranks of tokens probabilities across the input text \citep{gehrmann2019gltr}, the reward model used by the target LLM \citep{lee2024remodetect}, and the maximum mean discrepancy \citep{zhang2024detecting, song2025deep}. However, these approaches focus solely on classification and do not study statistical inference.

\item \textbf{ML-based} approaches leverage large human-written corpora available on the Internet, prompt the target LLM to generate the corresponding LLM-authored text, and then train classification models on both types of text for detection (see the upper right panel of Figure \ref{fig:relatedwork}). These methods can be further categorized into three types based on the classification models used: (i) classical ML models (e.g., decision trees, support vector machines); (ii) LLMs; and (iii) hybrids of (i) and (ii). Specifically, the first type extracts certain features from the input text and feed them into a classification model. For instance, the statistical measures described in statistics-based methods can serve as such features. Other features include the classical term frequency–inverse document frequency, unigram and bigram \citep{solaiman2019release}, the cross-entropy loss computed between the input text and the target LLM \citep{guo2024biscope}, and the semantic difference measure between the original input text and its LLM-rewritten version \citep{mao2024raidar}. The second type of methods uses LLMs directly as classification models. Examples of language models employed for this purpose include RoBERTa \citep{solaiman2019release, guo2023close}, BERT \citep{ippolito2020bert}, and DistilBERT \citep{mitrovic2023chatgpt}. This approach is well justified, as LLMs are inherently designed to process text, and the resulting model parameters can be fine-tuned on the dataset. The last type of methods uses the outputs of fine-tuned LLMs as input features for classical ML-based classification \citep{abburi2023generative}. Similar to statistics-based approaches, statistical inference is not considered in these methods.
\end{enumerate}
To summarize, all existing methods suffer from certain limitations. Watermarking-based approaches require knowledge of the hash function or RNG used by the target LLM, but they can conduct statistical inference to produce a $p$-value for uncertainty quantification. In contrast, most statistics- and ML-based methods do not require access to the hash function or RNG. But they typically cannot perform statistical inference, one exception being \citet{zhu2025reliably}. Finally, all these methods are model-specific: they are designed to detect text generated by a particular target LLM and do not generalize to other models. 

\subsection{Contribution}
This paper proposes a detector designed to overcome the limitations of existing methods. Our main contributions are listed below: 
\begin{enumerate}[leftmargin=*]
    \item We collect a dataset of over 10,000 texts, covering a range of domains from medical to legal documents, consisting of texts written from both humans and recent popular LLMs such as the GPT series, Grok, and Gemini. This dataset can be used to train classifiers for detecting LLM-generated text across different domains.  
    \item We train a detector using the collected dataset that is target LLM-agnostic and does not require access to the model's hash function or RNG. It achieves superior classification performance compared to existing statistics- and ML-based detectors, while also controlling type-I error and maintaining high power similar to watermarking-based methods.
    \item We deploy a publicly available website (see Figure~\ref{fig:website}) to host our detector, helping users detect LLM-generated text without downloading the model or retraining the detector on our data.
\end{enumerate}

\begin{figure}[t]
    \centering
    \includegraphics[width=1.0\linewidth]{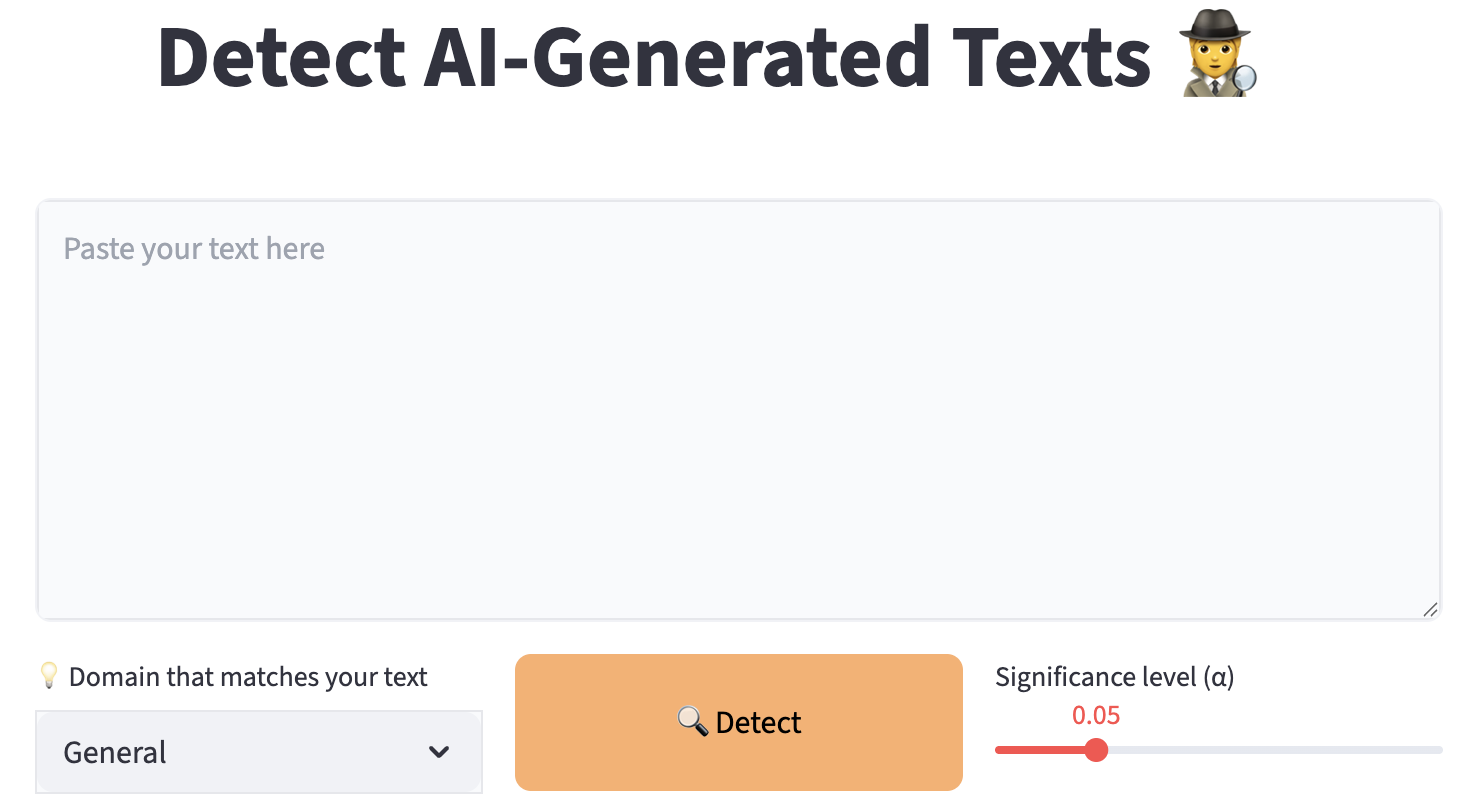}
    \caption{Website interface for our detector: users can enter text in the shaded area labeled “Paste your text here,” select the domain of the text (e.g., finance, law; by default, general), choose a significance lever $\alpha$, and click “Detect.” The detector then produce an output (see Figure~\ref{fig:website-bottom}). If not specified, the significance level is set to $0.05$.}\label{fig:website}
\end{figure}

\subsection{Paper organization}
The rest of the paper is organized as follows. Section~\ref{sec:preliminary} reviews the problem setup and introduces logits-based detectors, which are closely related to our approach. Section~\ref{sec:data} presents the dataset we collect for training the detector. Section~\ref{sec:method} describes our training methodology. Section~\ref{sec:experiments} reports numerical results evaluating our method on both in-distribution and out-of-distribution data. Section~\ref{sec:website} introduces our website. Finally, Section~\ref{sec:conclusion} concludes the paper and discusses directions for future research. 

\section{Preliminaries}\label{sec:preliminary}
\subsection{Problem setup}
In this section, we introduce the notation and formally define the problem. In natural language processing, a token is a basic unit of text, such as a word, subword, character, or punctuation mark. We refer to the set of all such tokens as a vocabulary, denoted by $\mathcal{V}$. A text passage can then be represented as a sequence of tokens, denoted by $\bm{X} = (X_1, \ldots, X_T)$, where $X_t \in \mathcal{V}$ denotes the $t$th token in the sequence, and $T$ represents the total number of tokens in the passage. Without loss of generality, we assume that all passages are of length $T$. This can be achieved by setting $T$ to the maximum length and padding shorter passages with zeros to reach the required length. 

Let $\mathbb{P}$ and $\mathbb{Q}$ represent the distribution functions over human-written and LLM-generated token sequences, respectively. Since each text $\bm{X}$ is a time series, both $\mathbb{P}$ and $\mathbb{Q}$ can be decomposed into sequences of conditional probability distributions $\{p_t\}_{t=1}^T$ and $\{q_t\}_{t=1}^T$, where each $p_t$ (and $q_t$) denotes the conditional distribution of $X_t$ given all preceding tokens, denoted by $\bm{X}_{<t}= (X_1, \cdots, X_{t-1})$, i.e., 
\begin{equation*}
    p_t(x_t|\bm{x}_{<t})=\mathbb{P}(X_t=x_t|\bm{X}_{<t}=\bm{x}_{<t}),\,\, q_t(x_t|\bm{x}_{<t})=\mathbb{Q}(X_t=x_t|\bm{X}_{<t}=\bm{x}_{<t}),
\end{equation*}
for any $x_t\in\mathcal{V}$, $\bm{x}_{<t}\in \mathbb{V}^{t-1}$.

Notice that neither $\mathbb{P}$ or $\mathbb{Q}$ is known. Furthermore, in contrast to most existing approaches \citep[see, e.g.,][]{mitchell2023detectgpt,bao2024fastdetectgpt}, we do not restrict $\mathbb{Q}$ to a particular target LLM. Instead, $\mathbb{Q}$ corresponds to a mixture distribution, representing a mixture of LLM distributions. Define the following pair of hypotheses: 
\begin{align}\label{eqn:testinghypotheses}
\mathcal{H}_0: \bm{X} \sim \mathbb{P} \textup{ versus } \mathcal{H}_1:  \bm{X} \sim \mathbb{Q}.
\end{align}
We aim to introduce a classification rule $S>c$, such that if $S>c$, we classify the text as LLM-generated. Here, $S$ denotes a statistical measure used to distinguish between $p$ and $q$, and $c$ denotes a classification threshold. Our goal is twofold:
\begin{enumerate}[leftmargin=*]
    \item To devise a more powerful statistic $S$, such that, when varying the the classification threshold $c$, the resulting classifier achieves larger area under the curve (AUC) than those proposed in the literature. 
    \item To compute a threshold $c$ such that the type-I error (i.e., when $\bm{X}\sim \mathbb{P}$, the probability of wrongly concluding $\bm{X}\sim \mathbb{Q}$) of the resulting classification rule is controlled at a pre-specified significance level $0<\alpha<1$.
\end{enumerate}

\subsection{Logits-based detector}\label{sec:logitsbaseddetector}
As reviewed in Section \ref{sec:related}, logits-based detectors are popular statistics-based detectors whose statistical measure is constructed based on the target LLM's logits. A logit is the raw, unnormalized score that an LLM $\mathcal{M}$ assigns to each possible next token before applying the softmax function to convert these scores into probabilities. Mathematically, let $\ell^{\mathcal{M}}_t(x_t|\bm{x}_{<t})$ denote the logit assigned by $\mathcal{M}$ to token $x_t$ given the preceding context $\bm{x}_{<t}$. Then the model's predicted probability for token $x_t$ is given by
\begin{equation}\label{eqn:logits}
    q_t^{\mathcal{M}}(x_t|\bm{x}_{<t})=\frac{\exp\{\tau^{-1}\ell^{\mathcal{M}}_t(x_t|\bm{x}_{<t})\}}{\sum_{x\in \mathcal{V}} \exp\{\tau^{-1}\ell^{\mathcal{M}}_t(x|\bm{x}_{<t})\}},
\end{equation}
where $\tau>0$ denotes the temperature parameter. 

A notable logits-based method is Fast-DetectGPT \citep{bao2024fastdetectgpt}, which is a computationally efficient variant of the earlier DetectGPT  \citep{mitchell2023detectgpt}. Fast-DetectGPT is built upon the following statistical measure to detect text generated by $\mathcal{M}$, 
\begin{equation}\label{eqn:fastdetect}
    \frac{\sum_t \log  q_t^{\mathcal{M}}(X_t| \bm{X}_{<t})-\sum_t \mathbb{E}_{\widetilde{X}_t \sim q_t^{\mathcal{S}}(\bullet | \bm{X}_{<t})} [\log q_t^{\mathcal{M}}(\widetilde{X}_t | \bm{X}_{<t})]}{\sqrt{\sum_t \textrm{Var}_{\widetilde{X}_t \sim q_t^{\mathcal{S}}(\bullet | \bm{X}_{<t})} (\log q_t^{\mathcal{M}}(\widetilde{X}_t | \bm{X}_{<t}))}},
\end{equation}
where $q_t^{\mathcal{S}}$ denotes the next-token prediction distribution of a \textit{sampling} model $\mathcal{S}$, used to sample $\widetilde{X}_t$ given $\bm{X}_{<t}$, and which may equal to or differ from the \textit{scoring} model $\mathcal{M}$.

The first term in the numerator of \eqref{eqn:fastdetect} is the log-likelihood of $\bm{X}$ under the target model $\mathcal{M}$. The second term in the numerator and the denominator serve as centering and normalization terms, respectively, ensuring that the statistic has approximately zero mean and unit variance when $\bm{X}$ is generated under the sampling model~$\mathcal{S}$. 

According to \eqref{eqn:logits}, it is immediate to see that the presence of the centering term cancels out the normalizing constant $\sum_{x\in \mathcal{V}} \exp\{\tau^{-1}\ell^{\mathcal{M}}_t(x|\bm{X}_{<t})\}$ that appears in the denominator of \eqref{eqn:logits}. As a result, the numerator of \eqref{eqn:fastdetect} becomes equal to  $$\frac{1}{\tau}\sum_t [\ell^{\mathcal{M}}_t(X_t|\bm{X}_{<t})-\mathbb{E}_{\widetilde{X}_t \sim q_t^{\mathcal{S}}(\bullet | \bm{X}_{<t})}\ell^{\mathcal{M}}_t(\widetilde{X}_t|\bm{X}_{<t})].$$
Furthermore, the normalization term in the denominator cancels the temperature parameter $\tau$, so that \eqref{eqn:fastdetect} becomes exactly the standardized version of the logits: 
\begin{equation}\label{eqn:logitsfast}
    S_{\textrm{Fast}}(\bm{X})=\frac{\sum_t \ell_t^{\mathcal{M}}(X_t| \bm{X}_{<t})-\sum_t \mathbb{E}_{\widetilde{X}_t \sim q_t^{\mathcal{S}}(\bullet | \bm{X}_{<t})} [\ell_t^{\mathcal{M}}(\widetilde{X}_t | \bm{X}_{<t})]}{\sqrt{\sum_t \textrm{Var}_{\widetilde{X}_t \sim q_t^{\mathcal{S}}(\bullet | \bm{X}_{<t})} (\ell_t^{\mathcal{M}}(\widetilde{X}_t | \bm{X}_{<t}))}}.
\end{equation}

\begin{figure}[t]
    \centering
    \includegraphics[width=0.95\linewidth]{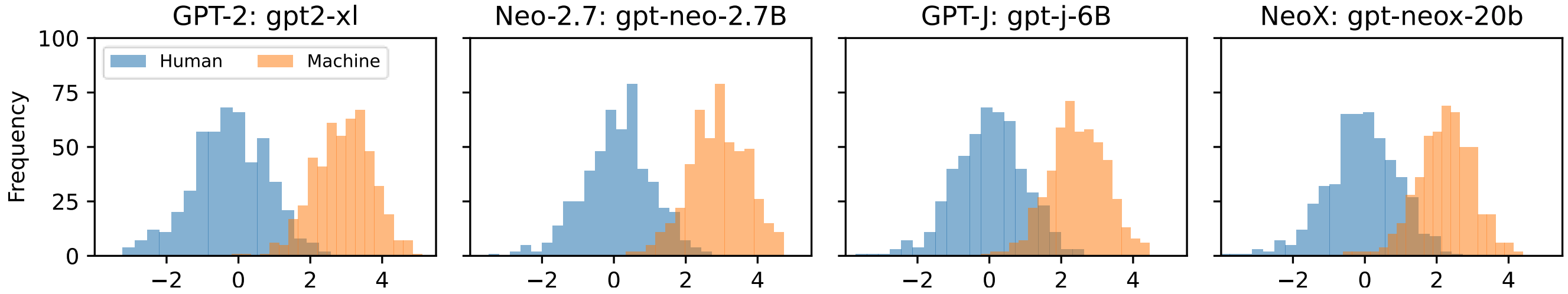}
    \caption{Histograms of the Fast-DetectGPT statistic evaluated on human-authored text and text generated by various LLMs, taken from \citet{bao2024fastdetectgpt}.}
    \label{fig:statdiff}
\end{figure}
The rationale for using \eqref{eqn:logitsfast} as the  statistic is that, LLM-generated text tends to yield larger values of this statistic compared to human-written text on average. To justify this observation, we present the following theorem: 
\begin{theorem}\label{thm:statdiff}
    Suppose the sampling model $\mathcal{S}$ is the same to the scoring model (i.e., the target model $\mathcal{M}$). Then as the temperature parameter $\tau\to 0^+$, we have
    \begin{equation*}
        \mathbb{E}_{\bm{X} \sim \mathbb{P}} S_{\textrm{Fast}}(\bm{X})\le \mathbb{E}_{\bm{X} \sim \mathbb{Q}^{\mathcal{M}}} S_{\textrm{Fast}}(\bm{X}),
    \end{equation*}
    where $\mathbb{Q}^{\mathcal{M}}$ denotes the probability distribution of text generated by the target model $\mathcal{M}$. 
\end{theorem}
The proof of Theorem \ref{thm:statdiff} is simple. When $\mathcal{S}=\mathcal{M}$, it suffices to show that 
\begin{equation}\label{eqn:someeqn}
    \sum_t \mathbb{E}_{\bm{X}\sim \mathbb{P}}\ell_t^{\mathcal{M}}(X_t| \bm{X}_{<t})-\sum_t \mathbb{E}_{\substack{\bm{X}_{<t}\sim \mathbb{P}\\\widetilde{X}_t \sim q_t^{\mathcal{M}}(\bullet | \bm{X}_{<t})}} [\ell_t^{\mathcal{M}}(\widetilde{X}_t | \bm{X}_{<t})]\le 0
\end{equation}
As the temperature parameter approaches $0$, sampling $\widetilde{X}_t$ from $q_t^{\mathcal{M}}$ concentrates all mass on the token that maximizes the logit $\ell_t^{\mathcal{M}}(\bullet|\bm{X}_{<t})$. Consequently, the left-hand-side of \eqref{eqn:someeqn} is always non-positive, regardless of the distribution $\mathbb{P}$. 

Theorem \ref{thm:statdiff} therefore establishes the result for sufficiently small temperatures. In practice, however, the temperature used by modern LLMs is often nonzero. Nevertheless, empirical evidence consistently shows that this statistic tends to take larger values on LLM-generated text than on human-written text; see Figure \ref{fig:statdiff} for illustrations. Consequently, when this statistic exceeds certain threshold, the input text can be classified as LLM-generated.

However, when the temperature is large, relying solely on logits is not sufficient to consistently distinguish between different distributions; see the ``Kingdom of Bit’’ example discussed in \citet{zhou2025adadetectgpt}. To address this limitation, \citet{zhou2025adadetectgpt} develop AdaDetectGPT, based on the following statistic
\begin{equation}\label{eqn:statada}
    S_{\textrm{Ada}}(\bm{X})=\frac{\sum_t w(\log  q_t^{\mathcal{M}}(X_t| \bm{X}_{<t}))-\sum_t \mathbb{E}_{\widetilde{X}_t \sim q_t^{\mathcal{S}}(\bullet | \bm{X}_{<t})} [w(\log q_t^{\mathcal{M}}(\widetilde{X}_t | \bm{X}_{<t}))]}{\sqrt{\sum_t \textrm{Var}_{\widetilde{X}_t \sim q_t^{\mathcal{S}}(\bullet | \bm{X}_{<t})} (w(\log q_t^{\mathcal{M}}(\widetilde{X}_t | \bm{X}_{<t})))}}.
\end{equation}
Comparing \eqref{eqn:statada} with \eqref{eqn:fastdetect}, we see that the two statistics share very similar structure. The only difference is that \eqref{eqn:statada} applies a one-dimensional witness function 
$w$ to the log-probabilities before performing the same centering and normalization steps. \citet{zhou2025adadetectgpt} propose to parameterize   this witness function $w$ using B-spline basis functions and to learn it adaptively from data by maximizing the AUC of the resulting detector. 

Their key observation lies in that the numerator of \eqref{eqn:statada} forms a martingale difference sequence with respect to the filtration $\{\sigma(\bm{X}_{<t})\}_{t\ge 1}$ where $\sigma(\bm{X}_{<t})$ stands for the $\sigma$-algebra generated by $\bm{X}_{<t}$. This enables the application of the classical martingale central limit theorem \citep[MCLT,][]{hall2014martingale} to derive closed-form expression for the false negative rate (FNR) and true negative rate (TNR) of the resulting classifier, which facilitates the derivation of an objective function for optimizing $w$. Theoretically, they establish statistical guarantees for their classifier, including lower bounds on its FNR and false positive rate (FPR), as well as upper bounds on its TNR and true positive rate (TPR).

However, using a one-dimensional witness function 
$w$ is insufficient to substantially improve classification accuracy. Moreover, neither Fast-DetectGPT nor AdaDetectGPT provides control over the type-I error. Both methods are also model-dependent: their test statistics are specific to a given target model and thus require the target model to be specified. In practice, however, users typically do not have such prior information -- they wish to determine whether a text is generated by an LLM at all, without knowing in advance which specific LLM might produce it. We will address these limitations in the next two sections.

\section{Data}\label{sec:data}
To address the aforementioned limitations of DetectGPT and AdaDetectGPT, we construct a large dataset consisting of both human-written text and LLM-generated text to train our detector. We detail both types of data in this section. 
\subsection{Human-written text}\label{sec:human-written text}

\textbf{Data sources.} Our preliminary study reveals that the performance of ML-based detectors depends heavily on the training data.  For instance, if the training dataset contains a large amount of text from one domain over others, then the resulting detector tends to achieve high classification accuracy on that domain while performing poorly on other underrepresented domains. This observation motivates us to collect training data that is as diverse as possible across a wide range of domains. 

Guided by this principle, we collect a large corpus of human-written texts across eight domains, listed alphabetically: (i) academia, (ii) finance, (iii) government, (iv) knowledge, (v) legislation, (vi) medicine, (vii) news, and (viii) user reviews. Table~\ref{tab:humanwrittentext} provides a summary of the datasets used for each domain.

\begin{table}
    \centering
    \begin{tabular}{|c|l|}\hline
     Academic & Paper abstracts from arXiv \citep{arxiv2024kaggle} \\
     (50000) & Research papers published in PubMed \citep{cohan2018discourse}\\ \hline
    \multirow{2}{*}{Finance} & Public Reddit discussions on finance \\
    & Financial Opinion Mining and Question Answering dataset \citep{maia2018www} \\
    (49254) & Bloomberg articles on finance from 2006 to 2013  \citep{josephgflowers2025financeinstruct}\\ \hline
    Government & \multirow{2}{*}{U.S. government reports and expert-written summaries \citep{efficient2021huang}}\\ 
    (38887) & \\ \hline
    Knowledge & \multirow{2}{*}{Wikipedia articles prior to April 2022 \citep{wikidump}} \\ 
    (49251) & \\ \hline
    Legislation & \multirow{2}{*}{U.S. and European union legislation documents \citep{multieurlex2021chalkidis,lexglue2021chalkidis}}  \\
    (49987) & \\ \hline
    Medicine & Public patient notes from open-source datasets \citep{khandekar2024medcalc}\\
    (44917) & Expert answers to medical entrance exam questions \citep{medmcqa2022flores}\\ \hline
    News & \multirow{2}{*}{Articles from CNN and BBC \citep{greene2006practical,Narayan2018DontGM}} \\ 
    (49954) & \\ \hline
    User reviews & \multirow{2}{*}{Reviews from Yelp, ImDB and Amazon \citep{maas2011wordvector,zhang2015character}}\\ 
    (47902) & \\ \hline
    \end{tabular}
    \caption{Summary of human-written text datasets by domain. The numbers in parentheses indicate the total number of texts collected for each category.}
    \label{tab:humanwrittentext}
\end{table}

These texts are collected from publicly available platforms such as Kaggle, GitHub, and Hugging Face. In each domain, we gathered data from at least two different sources to ensure diversity, resulting in a total of 18 datasets. To avoid contamination by LLM-generated content, we retain only datasets written prior to November 2022 (the public release date of ChatGPT) and discard all data released thereafter.

\textbf{Data processing.} After collection, all texts were processed to ensure high quality before being fed into our algorithm to train the detector. When constructing the training dataset, we randomly sampled an equal number of texts from each domain to avoid overrepresenting any domain. We next apply the following processing steps to each text: removing control characters, trimming extra spaces and newlines, deleting special tokens, and eliminating duplicates (see Section~\ref{sec:data-human} in the Supplementary Material for details of each step).  

Following this process, each domain contains over 38{,}000 texts, resulting in a combined dataset of more than 370{,}000 human-written texts. We summarize the total number of texts per domain in Table~\ref{tab:humanwrittentext}. Figure~\ref{fig:n-word} presents boxplots of the number of words in human-written texts for each domain. It can be seen that the average text length ranges from a few dozen to several hundred words. This variation reflects the diversity of writing styles and formats across domains.

\begin{figure}[t]
    \centering
    \includegraphics[width=0.83\linewidth]{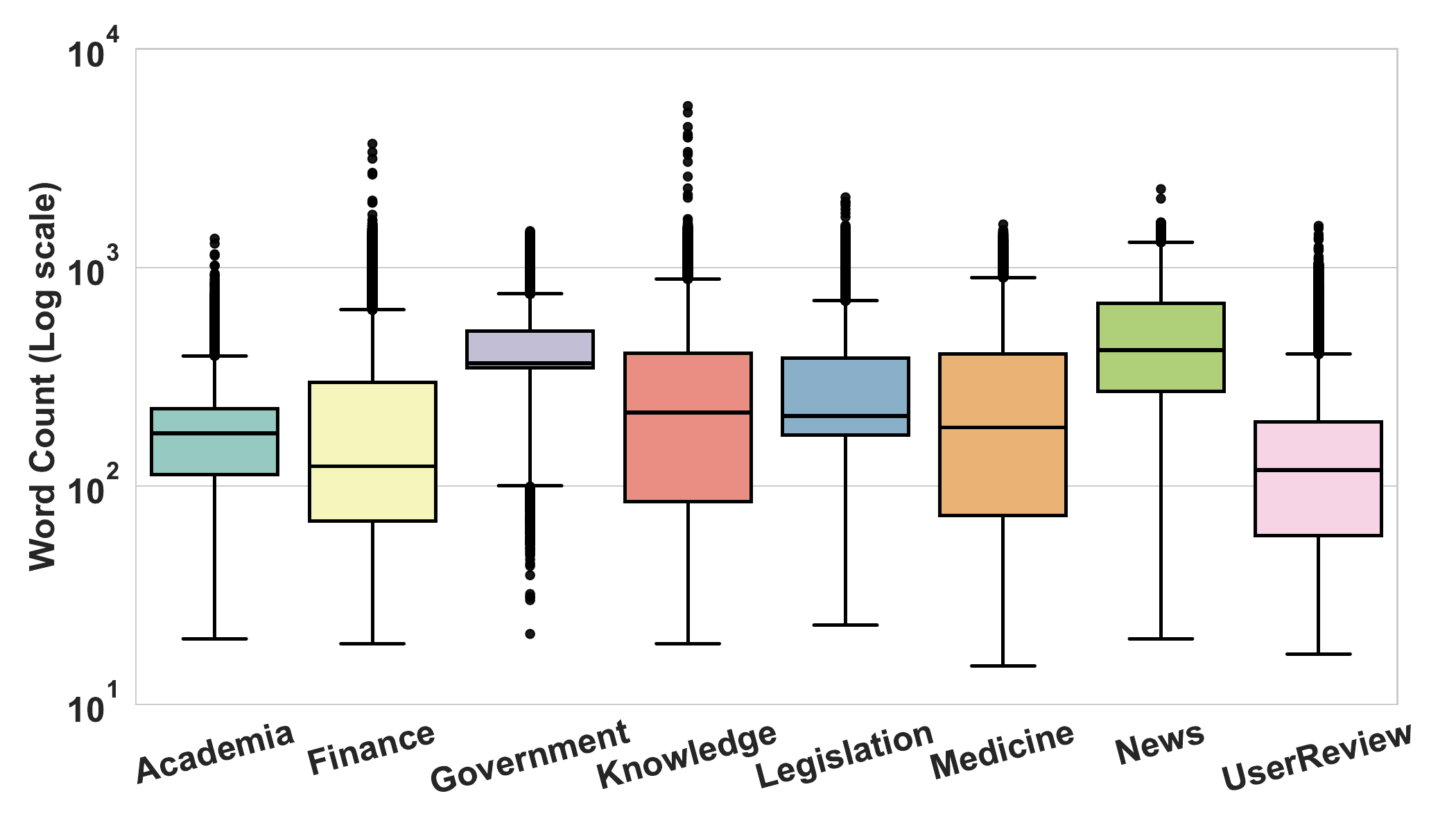}
    \vspace*{-10pt}
    \caption{Boxplots of the number of words in human-written texts for each domain in the collected dataset.}
    \label{fig:n-word}
\end{figure}

\subsection{LLM-generated text}
We next prompt the LLM to generate LLM-authored text based on the aforementioned  human's texts. The purpose of generating these texts is twofold: (i) they are used to train our detector, and (ii) they are used to evaluate the detection power and accuracy of our detector relative to existing methods. To ensure that the resulting detector is model-agnostic, we generate texts using a diverse set of widely used LLMs, including GPT-4o \citep{hurst2024gpt}, Claude \citep{anthropic2024claude}, Gemini \citep{comanici2025gemini}, and Grok \citep{xAI2023Grok}. In addition, we employ a variety of prompts to obtain diverse outputs from the LLMs based on the same human-written inputs, including rewriting, polishing and expanding, and summarization. Further details of the data generating procedure are provided in Section~\ref{sec:data-machine} of the Supplementary Material.

\section{Methodology}\label{sec:method}
We first describe the procedure for training the statistic $S$ used by our detector (Section \ref{sec:fine-tune}). We next detail the determination of the threshold $c$ to ensure valid hypothesis testing (Section \ref{sec:testing}).
\subsection{Training}\label{sec:fine-tune}
Our detector builds upon Fast-DetectGPT. As discussed in Section~\ref{sec:logitsbaseddetector}, relying solely on log-probabilities for constructing the statistical measure can be suboptimal when the temperature is large. AdaDetectGPT partly mitigates this issue by applying a  witness function to the log-probabilities; however, the expressive power of a one-dimensional witness function remains limited. To address these limitations, we consider the following statistic: 
\begin{equation}\label{eqn:ourstatistic}
    S(\bm{X}) \coloneqq\frac{\sum_{t=1}^T [ w_t(X_t|\bm{X}_{<t})-\mathbb{E}_{\widetilde{X}_t \sim q_t(\bullet | \bm{X}_{<t})} \{ w_t(\widetilde{X}_t | \bm{X}_{<t})\}]}{\sqrt{\sum_t \textrm{Var}_{\widetilde{X}_t \sim q_t(\bullet |\bm{X}_{<t})} ( w_t(\widetilde{X}_t | \bm{X}_{<t}))}},
\end{equation}
for a sequence of functions $\{w_t\}_t$ adaptively learned from our constructed dataset described in Section 
\ref{sec:data}. 

The key difference between our statistical measure and that of AdaDetectGPT lies in the flexibility of the learned functions. In \eqref{eqn:ourstatistic}, each $w_t$ can be an arbitrary function of both the preceding tokens $\bm{X}_{<t}$ and the current token $X_t$, and it is allowed to vary over $t$. In contrast, AdaDetectGPT employs a time-invariant witness function $w$ that takes as input only a one-dimensional log-probability. Indeed, by setting $w_t = w \circ \log q_t^{\mathcal{M}}$, our statistic reduces to $S_{\mathrm{Ada}}(\bm{X})$. More generally, employing a sequence of functions $\{w_t\}_{t=1}^T$ substantially increases the expressive power of the resulting statistical measure. We will demonstrate this advantage empirically in Section \ref{sec:experiments}.

We next define an objective function for learning $\{w_t\}_t$ that maximizes the classification accuracy of the resulting statistic. Since this is inherently a binary classification problem, a natural objective to consider is the AUC. Notice that maximizing AUC is equivalent to maximizing the TNR of the detector at any fixed FNR. We follow \citet{zhou2025adadetectgpt} to derive our learning objective. 

Specifically, similar to \eqref{eqn:statada}, when $\bm{X}\sim \mathbb{Q}$, for any sequence $\{w_t\}_t$, the numerator of \eqref{eqn:ourstatistic} forms a martingale difference sequence with respect to $\{\sigma(\bm{X}_{<t})\}_{t\ge 1}$. By setting the classification threshold $c$ to the upper $\alpha$th quantile of the standard normal distribution (denoted by $z_{\alpha}$), it follows from the MCLT that the FNR of the resulting detector $S(\bm{X})$ is asymptotically controlled at level $\alpha$. Applying the MCLT again allows us to derive a lower bound on its corresponding TNR, which we formalize in the following theorem.
\begin{theorem}\label{thmLw}
    Under conditions specified in Section~\ref{sec:proof} of the Supplementary Material, the TNR of our detector, at an FNR level of $\alpha$, is asymptotically lower bounded by $\min\{\alpha+\phi(z_{\alpha})L_w,1-\alpha\}$, where $\phi$ denotes the probability density function of the standard normal distribution, and 
    \begin{equation*}
        L_w=\frac{\sum_{t} [\mathbb{E}_{\substack{\widetilde{X}_t \sim q_t(\bullet | \bm{X}_{<t})\\ \bm{X}_{<t} \sim \mathbb{P}}} \{ w_t(\widetilde{X}_t | \bm{X}_{<t})\}-\mathbb{E}_{\substack{\widetilde{X}_t \sim p_t(\bullet | \bm{X}_{<t})\\ \bm{X}_{<t} \sim \mathbb{P}}}\{w_t(X_t|\bm{X}_{<t})\}]}{\sqrt{\sum_t \textrm{Var}_{\substack{\widetilde{X}_t \sim q_t(\bullet |\bm{X}_{<t})\\ \bm{X}_{<t} \sim \mathbb{P}}} \{ w_t(\widetilde{X}_t | \bm{X}_{<t})\}}}.
    \end{equation*}
\end{theorem}
We make a few remarks. First, $L_w$ is a scalar whose value depends on $\{w_t\}_t$. Second,  maximizing the lower bound in Theorem \ref{thmLw} is equivalent to maximizing $L_w$. Third, because $L_w$ does not depend on the FNR level $\alpha$, optimizing $L_w$ simultaneously maximizes the TNR lower bound for all $\alpha$. Consequently, the maximizer of $L_w$ also maximizes a lower bound on the AUC of the resulting detector. Finally, while it is possible to derive a closed-form expression for the TNR itself (rather than a lower bound), such an expression generally depends on $\alpha$. Directly maximizing it would produce an optimizer tailored to a specific FNR level, without guaranteeing optimal performance across other levels. In contrast, maximizing the lower-bound-based objective $L_w$ yields an ``$\alpha$-uniform'' optimizer.

Based on the above discussion, it is natural to set $L_w$ to the objective function for optimization. We make a few modifications to further simplify the optimization. First, we replace the expectation $\bm{X}_{<t}\sim \mathbb{P}$ in the first term of the numerator of $L_w$ with $\bm{X}_{<t}\sim \mathbb{Q}$. The resulting numerator then simplifies to 
\begin{equation}\label{eqn:numeratordiff}
     \mathbb{E}_{\bm{X}\sim \mathbb{P}} \Big[\sum_t w_t(X_t|\bm{X}_{<t})\Big]-\mathbb{E}_{\bm{X}\sim \mathbb{Q}} \Big[\sum_t w_t(X_t|\bm{X}_{<t})\Big]. 
\end{equation}
Notice that maximizing \eqref{eqn:numeratordiff} is closely related to the maximum mean discrepancy measure widely studied in machine learning \citep{gretton2012kernel}. 
Additionally, both expectations in \eqref{eqn:numeratordiff} can be approximated by empirical averages computed from the human-written and LLM-generated texts in our constructed dataset. 

Second, following this modification of the numerator, we adjust the denominator of $L_w$ accordingly as the standard deviation of the empirical averages. Specifically, we set the denominator to $\sqrt{\textrm{Var}_{\bm{X}\sim \mathbb{P}}(\sum_t w_t(X_t|\bm{X}_{<t})+\textrm{Var}_{\bm{X}\sim \mathbb{Q}}(\sum_t w_t(X_t|\bm{X}_{<t}))}$, which yielding the following two-sample $t$-test-type objective function, 
\begin{align}\label{eqn:finalobjective}
    \frac{\widehat{\mathbb{E}}_{\bm{X} \sim \mathbb{Q}} \{\sum_t w_t(X_t|\bm{X}_{<t})\} - \widehat{\mathbb{E}}_{\bm{X} \sim \mathbb{P}} \{\sum_t w_t(X_t|\bm{X}_{<t})\}}{\sqrt{\widehat{\textrm{Var}}_{\bm{X}\sim \mathbb{P}}(\sum_t w_t(X_t|\bm{X}_{<t})+\widehat{\textrm{Var}}_{\bm{X}\sim \mathbb{Q}}(\sum_t w_t(X_t|\bm{X}_{<t}))}},
\end{align}
where $\widehat{\mathbb{E}}$ and $\widehat{\textrm{Var}}$ denote the empirical average and sampling variance estimators computed from the constructed dataset. 

To optimize \eqref{eqn:finalobjective}, we need to specify a function class for $\{w_t\}_t$. Here, we parameterize $\{w_t\}_t$ using a base language model $\{q_t^{\mathcal{B}}\}_t$, since both $q_t^{\mathcal{B}}$ and $w_t$ take $X_t$ and $\bm{X}_{<t}$ as input and output a scalar value. We then fine-tune the parameters of $\{q_t^{\mathcal{B}}\}_t$ on our constructed dataset to maximize the objective in \eqref{eqn:finalobjective}. Finally, we plug in the fine-tuned model for $\{w_t\}_t$ in \eqref{eqn:ourstatistic} to obtain our statistical measure.

\subsection{Testing}\label{sec:testing}
In this section, we discuss how to choose the threshold $c$ to control the type-I error (e.g., FPR) for the testing hypotheses in \eqref{eqn:testinghypotheses}. Unlike the FNR, which is evaluated under $\bm{X}\sim \mathbb{Q}$, and for which the numerator of \eqref{eqn:ourstatistic} forms a martingale difference sequence, enabling the use of the MCLT to characterize its asymptotic distribution, the FPR is evaluated under  $\bm{X}\sim \mathbb{P}$. In this case, the asymptotic distribution of the test statistic is considerably more challenging to analyze.

To address this challenge, we estimate the null distribution empirically using our collected human-written texts, denoted by $\{\bm{X}^{(j)}\}_{j=1}^m$. For each $\bm{X}^{(j)}$, we compute our statistic $\bm{S}(\bm{X}^{(j)})$. Given a new text $\bm{X}$ to be tested, 
we compute its $p$-value as
\begin{align*}
    p\textup{-value} = \frac{1 + \sum_{j=1}^m \mathbb{I}\!\left(S(\bm{X}) < S(\bm{X}^{(j)}) \right)}{1 + m},
\end{align*}
and the threshold $c$ as the largest value of  $S(\bm{X})$ such that the resulting $p$-value is no larger than a pre-specified significance level $0<\alpha<1$. We reject the null and conclude that $\bm{X}$ is LLM-generated if the $p$-value is no larger than $\alpha$, or equivalently, if $S(\bm{X})>c$. 

Theoretically, the following theorem establishes the  validity of this procedure. 
\begin{theorem}\label{thm:typeI}
    The proposed test asymptotically controls the type-I error as $m\to \infty$. 
\end{theorem}
Empirically, we observe that the distribution of $S(\bm{X})$ varies substantially across different categories of human-written text. To account for this heterogeneity, when testing a new text $\bm{X}$, we ask the user to specify which of the eight categories described in Section~\ref{sec:data} the text belongs to, and we compute the corresponding $p$-value and classification threshold using the human-written texts from that category. When no domain information is available, we adopt a conservative strategy by reporting the maximum $p$-value across all eight categories. Figure~\ref{fig:null-distr} visualizes the null distributions of $S(\bm{X})$ for human-written texts across different categories.

\section{Real data analysis}\label{sec:experiments}
We evaluate the finite sample performance of the proposed detector in this section. We employ both in-distribution data, drawn from the eight categories described in Section~\ref{sec:data} and evaluated using sample splitting (Section~\ref{sec:experiments-collected}), and out-of-distribution data from external datasets (Section~\ref{sec:dataexternal}). Finally, we report the computational cost in Section~\ref{sec:computingtime}.

\subsection{In-distribution evaluation}\label{sec:experiments-collected}
We begin by evaluating the performance of our detector on the datasets collected in Section~\ref{sec:data}, comparing it against 9 representative baseline detectors from the literature, covering both statistics- and ML-based approaches: (i) \textit{Likelihood} \citep{gehrmann2019gltr};  (ii) \textit{Log-rank} \citep{gehrmann2019gltr}; (iii) Log Rank Ratio \citep[\textit{LRR}, ][]{su2023detectllm}; (iv) Fast-DetectGPT \citep[\textit{FDGPT},][]{bao2024fastdetectgpt}; (v) \textit{Binoculars} \citep{hans2024spotting}; (vi) \textit{RoBERTa} \citep{solaiman2019release}; (vii) \textit{RADAR} \citep{hu2023radar}; (viii) Imitate Before Detection \citep[\textit{ImBD},][]{chen2025imitate}; (ix) \textit{AdaDetectGPT} \citep{zhou2025adadetectgpt}. 
Since these detectors are primarily designed for classification rather than statistical inference, they typically construct a statistical measure $S$ without specifying a classification threshold $c$ for controlling the type-I error. Consequently, we evaluate their performance using the AUC, which assesses the quality of $S$ independent of $c$. 

To ensure a fair comparison, we split the eight data categories introduced in Section~\ref{sec:data} into training and testing subsets. 
The same training data is used to train all methods, and the same testing data is used to compute AUC scores. 
Furthermore, for all approaches that require sampling (e.g., FastDetectGPT and AdaDetectGPT), we use the same sampling model $q^{\mathcal{S}}$ when constructing the statistical measure.

\begin{figure}[t]
    \centering
    \includegraphics[width=1.0\linewidth]{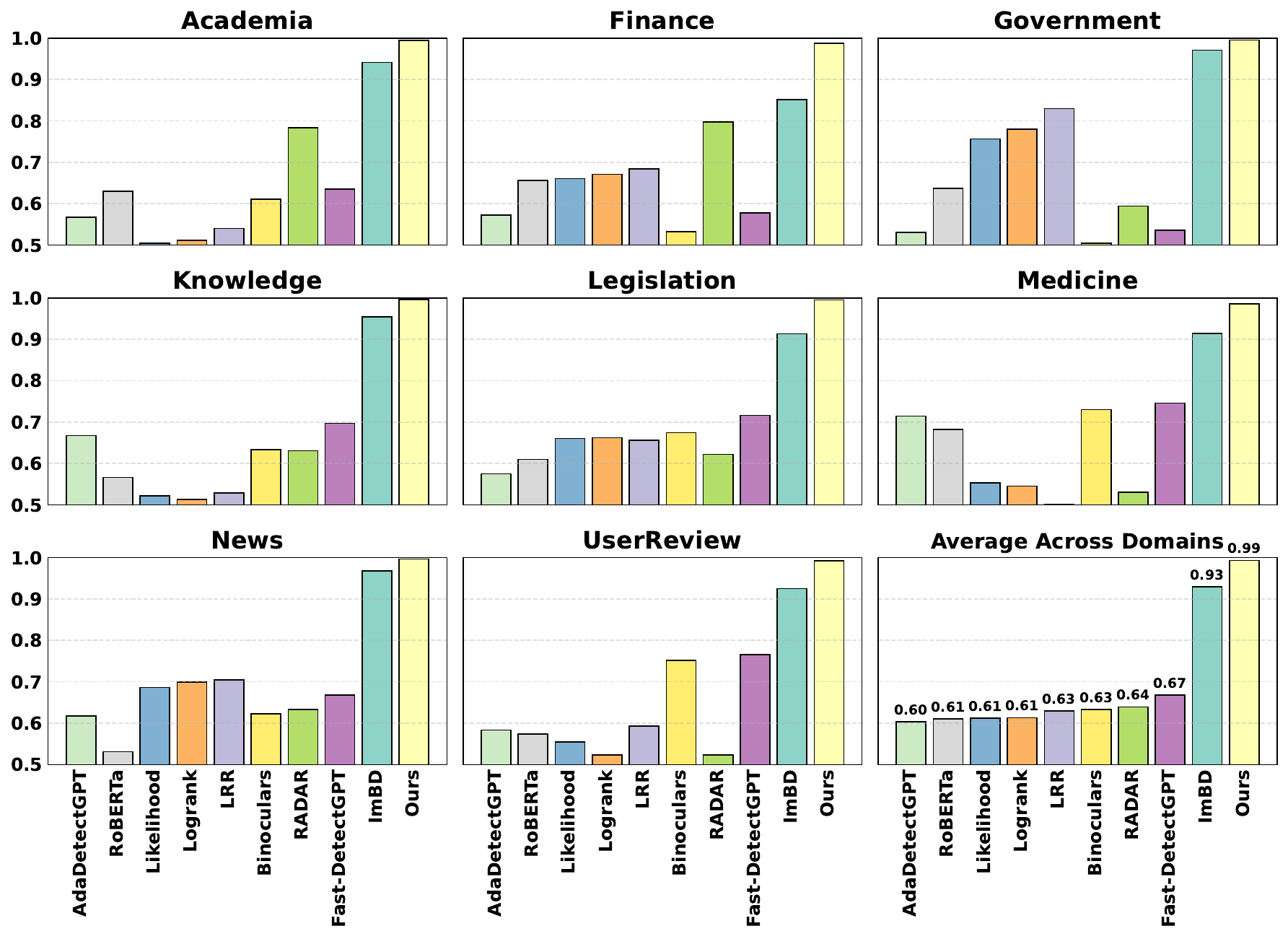}
    \caption{AUCs of various detectors when trained and evaluated on the eight domains of data described in Section~\ref{sec:data}. Each panel reports the AUC for one domain. The right bottom panel reports the average AUC across the eight domains.}
    \label{fig:roc-nice}
\end{figure}

We report the AUCs of these detectors when trained and evaluated on different data domains in Figure~\ref{fig:roc-nice}. It can be seen that: 
(i) the AUC of our proposed detector is consistently close to 1.0 across all domains, outperforming all baseline methods, with an average AUC across domains over 0.99; 
(ii) \textit{ImBD} achieves the second-best performance, but our method substantially outperforms it, with relative improvements \(\frac{\text{Our AUC} - \text{ImBD's AUC}}{1 - \text{ImBD's AUC}}\) exceeding 90\% in most cases; 
and (iii) all other baselines perform significantly worse than \textit{ImBD}.

Finally, we investigate the proposed procedure for selecting the classification threshold $c$ by evaluating the type-I error (i.e., FPR) and power (i.e., TPR) of our detector  at three significance levels, $\alpha \in \{0.01, 0.05, 0.1\}$. 
The empirical type-I error rates and power across the eight data categories are reported in Tables~\ref{tab:type1error} and~\ref{tab:power}, respectively. As shown in these tables, the type-I error rates are well controlled at the nominal levels, while the empirical power is close to $1$ in most cases. These results are consistent with our findings in Figure \ref{fig:roc-nice} and demonstrate the effectiveness of the proposed detector for statistical inference. 

\begin{table}[t]
    \centering
    \caption{Type~I error of our method on eight categories of human-written texts at three nominal significance levels $\alpha$.}
    \resizebox{\textwidth}{!}{%
    \begin{tabular}{l|cccccccc}
    \toprule
    $\alpha$ & Academia & Finance & Government & Knowledge & Legislation & Medicine & News & UserReview \\
    \midrule
0.01 & 0.000 & 0.000 & 0.013 & 0.007 & 0.013 & 0.013 & 0.027 & 0.000 \\
0.05 & 0.007 & 0.020 & 0.053 & 0.047 & 0.033 & 0.040 & 0.067 & 0.060 \\
0.10 & 0.053 & 0.073 & 0.093 & 0.073 & 0.080 & 0.080 & 0.113 & 0.100 \\
    \bottomrule
    \end{tabular}
    }
    \label{tab:type1error}
\end{table}

\begin{table}[t]
    \centering
    \caption{Power of our method on eight categories of LLM-generated texts at three nominal significance levels $\alpha$.}
    \resizebox{\textwidth}{!}{%
    \begin{tabular}{l|cccccccc}
    \toprule
    $\alpha$ & Academia & Finance & Government & Knowledge & Legislation & Medicine & News & User Reviews \\
    \midrule
0.01 & 0.933 & 0.827 & 0.987 & 0.947 & 0.900 & 0.873 & 0.973 & 0.953 \\
0.05 & 0.967 & 0.933 & 0.993 & 0.980 & 0.973 & 0.960 & 0.993 & 0.973 \\
0.10 & 0.980 & 0.953 & 0.993 & 0.987 & 0.987 & 0.960 & 0.993 & 0.973 \\
    \bottomrule
    \end{tabular}
    }
    \label{tab:power}
\end{table}

\subsection{Out-of-distribution evaluation}\label{sec:dataexternal}
In this section, we investigate the performance the proposed detector on an external benchmark dataset, RAID\footnote{\url{https://huggingface.co/datasets/liamdugan/raid}} \citep{dugan2024raid}. This dataset is widely used for benchmarking LLM-generated text detectors in the literature. It differs from our collected dataset described in Section~\ref{sec:data} in three aspects. 
First, the human-written texts of the two datasets are different in content. 
Figure~\ref{fig:external} visualizes the distribution of word counts in human-written texts from RAID and our collected dataset. 
While the average text length in RAID is similar to ours (and slightly longer), its variability is substantially smaller. 
Second, the LLM-generated texts in RAID are not produced by the same set of language models used in our dataset. 
Finally, RAID additionally includes texts generated under 11 adversarial attack settings and 4 LLM decoding strategies \citep{dugan2024raid}.

\begin{figure}[t]
    \centering
    \includegraphics[width=0.5\linewidth]{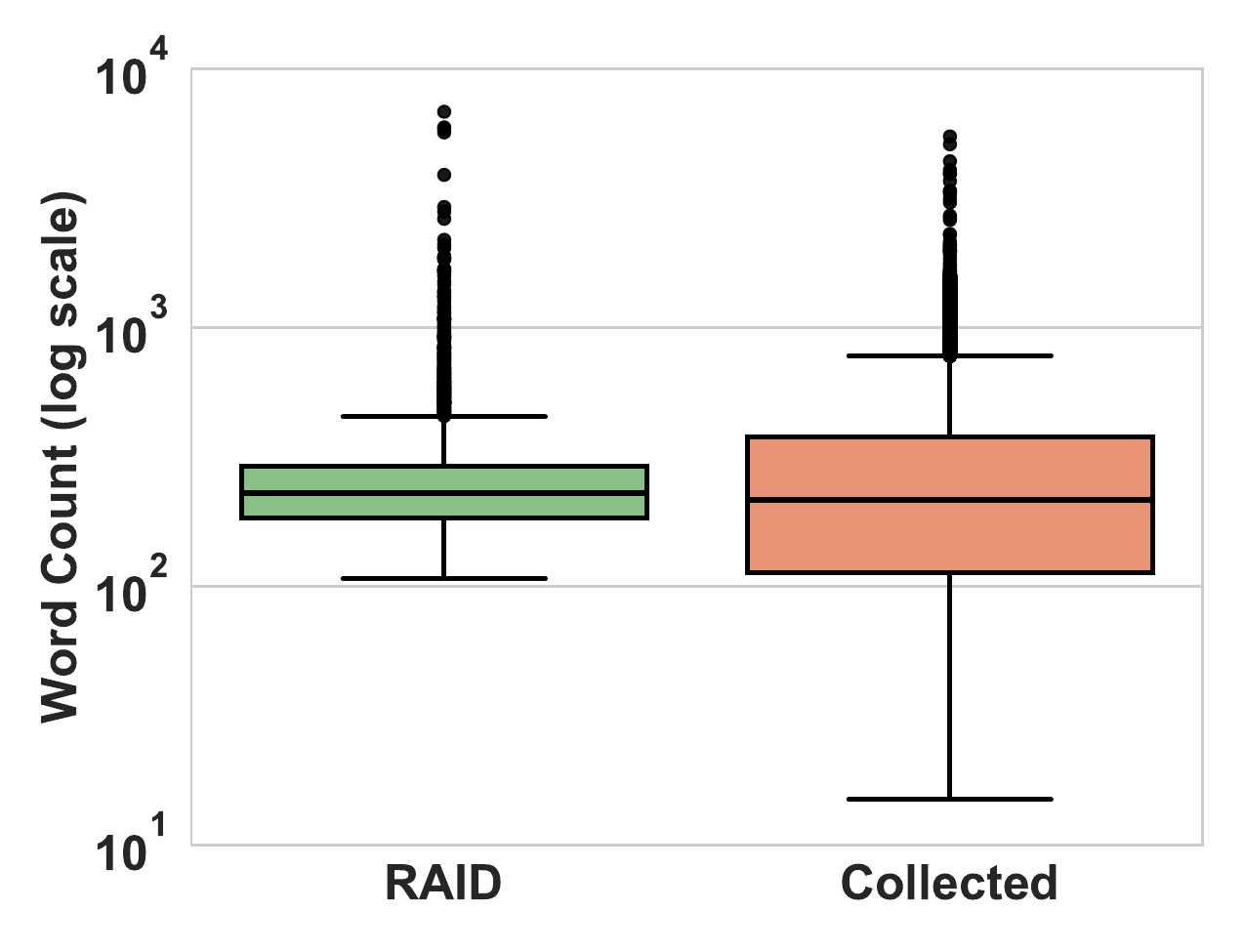}
    \vspace*{-10pt}
    \caption{Distribution of word counts for human-written texts in RAID and our collected dataset.}
    \label{fig:external}
\end{figure}

When training our detector for evaluation on RAID, we do not apply the sample splitting procedure described in Section~\ref{sec:experiments-collected}. 
Instead, we use all the data for detection. Table~\ref{tab:external} reports the empirical type-I error rate and power of our detector. As shown in the table, the type-I error rates are generally well controlled. 
For $\alpha = 0.01$ and $0.05$, the empirical type-I error is at or below the nominal significance level, with only a slight inflation when $\alpha = 0.1$. Additionally, the empirical power increases with $\alpha$, reaching nearly $80\%$ at $\alpha = 0.05$ and exceeding $90\%$ at $\alpha = 0.1$. These results are particularly impressive given that the LLM-generated texts in RAID are produced using different decoding temperatures than those used during training and under adversarial settings. They indicate that our procedure remains reliable for statistical inference under distribution shift, where the testing data differs from the training data.

\begin{table}
    \centering
    \caption{Empirical type-I error rate and power of the proposed detector on the 2000 human-written texts and 2000 GPT-4 generated texts in the RAID dataset.}\label{tab:external}
    {\footnotesize\begin{tabular}{l|ccc}
    \toprule
        $\alpha$ & 0.01 & 0.05 & 0.10 \\
    \midrule
        Type-I error rate & 0.007 & 0.050 & 0.127 \\
        Empirical power   & 0.450 & 0.780 & 0.906 \\
    \bottomrule
    \end{tabular}}
\end{table}

Finally, we compare our detector against the same set of baselines considered in Section~\ref{sec:experiments-collected} in terms of classification accuracy. Again, we use AUC as the evaluation criterion. The results are summarized in Table~\ref{tab:external-roc}. It can be seen that our detector still achieves the highest AUC, over 0.95. In contrast, the AUC of the best baseline detector reaches no higher than 0.88. Finally, it is worthwhile to note that although \textit{ImBD} performs desirably in Section \ref{sec:experiments-collected} when the training and test data follow the same distribution (Figure~\ref{fig:roc-nice}), its AUC drops substantially to below 0.8 on this external dataset. 
These results demonstrate that our approach is substantially more robust to distributional shifts between training and testing data than baseline detectors such as \textit{ImBD}.

\begin{table}[t]
\centering
\caption{AUCs of various detectors evaluated on RAID.}
\label{tab:external-roc}
\begin{tabular}{lccccc}
\toprule
 & AdaDetectGPT & Binoculars & Fast-DetectGPT & \textbf{Ours} & ImBD \\
\hline
AUC & 0.688 & 0.875 & 0.862 & \cellcolor{myblue}0.954 & 0.771 \\
\midrule
 & Likelihood & Logrank & LRR & RoBERTa & RADAR \\
\hline
AUC & 0.831 & 0.817 & 0.755 & 0.555 & 0.867 \\
\bottomrule
\end{tabular}
\end{table}

\subsection{Computational cost}\label{sec:computingtime}
We report the runtime (in seconds) and memory usage (in gigabytes) of the proposed detection procedure as a function of the number of tokens in the input text in Figure~\ref{fig:computation}. We make the following observations. First, the procedure typically completes within a few seconds. Meanwhile, the runtime increases approximately quadratically with the number of tokens, which is consistent with the computational complexity of Transformer-based architectures, where self-attention mechanisms compute pairwise interactions between all tokens in the input sequence \citep{vaswani2017attention}. Second, memory usage remains below 8~GB in most cases and grows approximately linearly with the number of tokens. (iii) The input length in our evaluation ranges from roughly 20 to over $2^{12}=4096$ tokens. Together, these results indicate that our method can comfortably handle both short and moderately long documents, such as short essays and reports, news headlines and articles on standard GPU hardware.

\begin{figure}[t]
    \centering
    \includegraphics[width=0.8\linewidth]{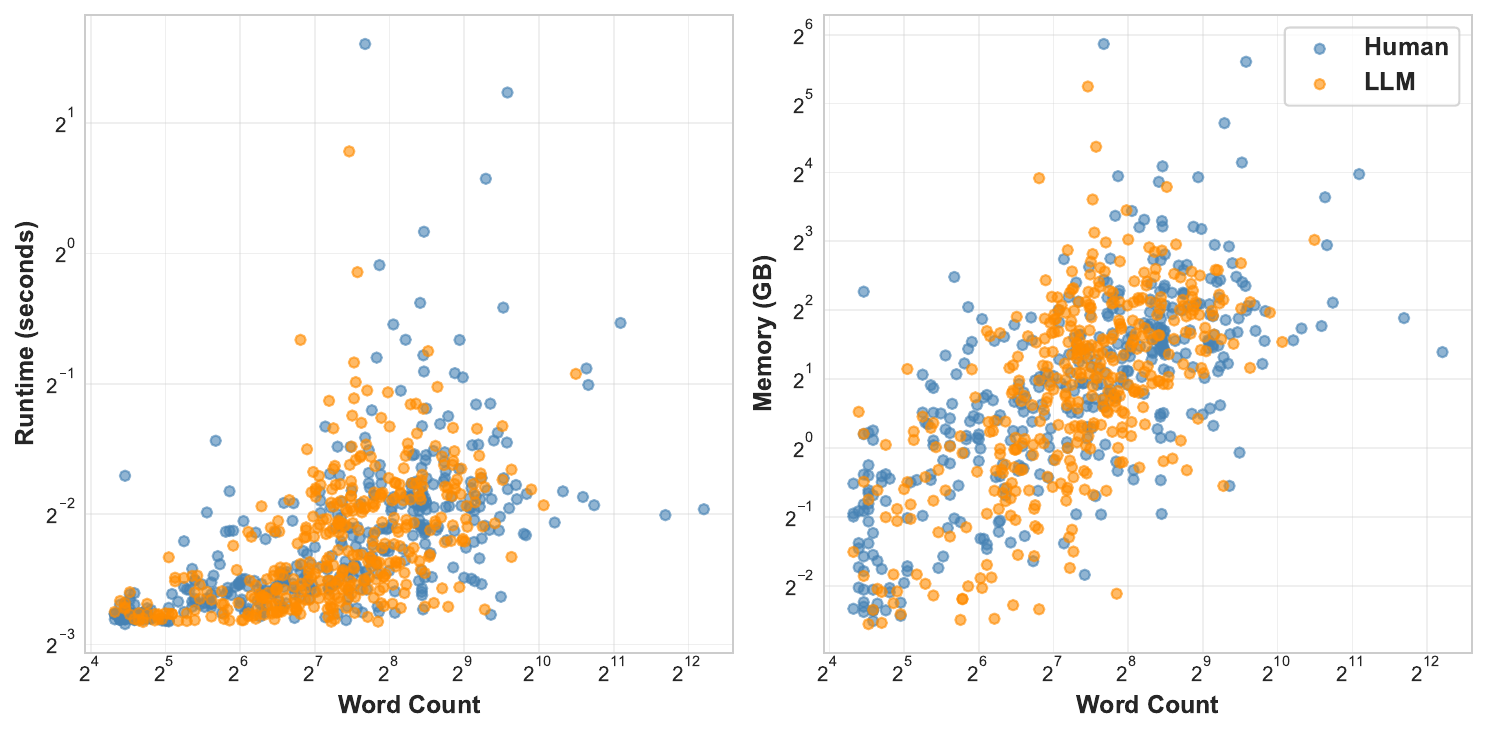}
    \caption{Left panel: the runtime of our method (on the $y$-axis) versus the number of tokens in the text to be detected (on the $x$-axis). Right panel: memory usage (on the $y$-axis) versus number of tokens. The runtime is measured in seconds, and memory usage in gigabytes (GB).}
    \label{fig:computation}
\end{figure} 

\section{Case studies}\label{sec:website}
We provide a publicly accessible website\footnote{\url{https://huggingface.co/spaces/stats-powered-ai/StatDetectLLM}}
 that hosts our trained detector. This serves two purposes. First, it allows users to apply our detector directly without downloading the model locally or retraining it using our data. Second, LLMs evolve and are updated rapidly. Detectors that are effective for LLMs released a few years ago may become outdated as those models are replaced or updated. By hosting our detector online, we plan to regularly update the set of popular LLMs used to generate training data and to retrain our detector accordingly, ensuring that it keeps pace with the rapid evolution of LLMs.

We next describe how to use our detector; a snapshot of our website is shown in Figure~\ref{fig:website}. The upper panel specifies the input provided by users and consists of the following components:
\begin{itemize}[leftmargin=*]
    \item \textbf{Input box} (grey): Users can paste the text to be detected here. 
    \item \textbf{Domain list} (grey): Users can specify the domain of the input text. The eight domains described in Section~\ref{sec:data} are available. If no such information is specified, the default option ``General'' is used, and our detector reports the maximum $p$-value across all eight domains to control the type-I error (see Section~\ref{sec:testing} for details).
    \item \textbf{Significance level $\alpha$ slider} (red): Users can adjust the slider to select an appropriate significance level between $0.01$ and $0.20$. If no value is specified, the default significance level is set to $0.05$.
    \item \textbf{Detect button} (orange): Clicking this button processes the input text, where our detector (i) computes the statistic in \eqref{eqn:ourstatistic}, (ii) uses the selected domain to calculate the corresponding $p$-value, and (iii) draws a conclusion based on the user-specified significance level.
\end{itemize}

\begin{figure}[t]
    \centering
    \includegraphics[width=0.9\linewidth]{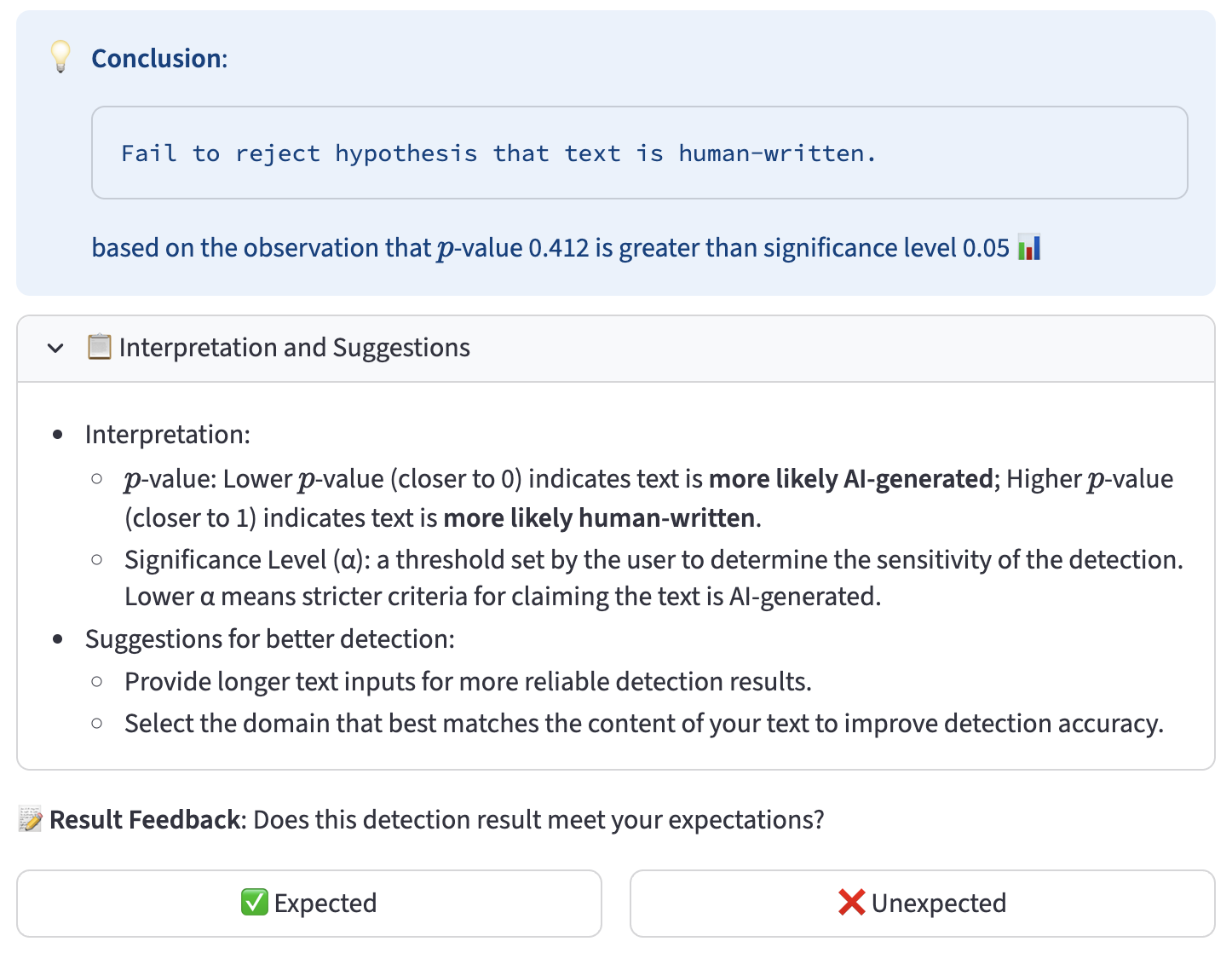}
    \caption{The outputs of detector is a conclusion panel indicating whether the text is human-written or LLM-generated,  along with the corresponding $p$-value.}\label{fig:website-bottom}
\end{figure}
After clicking the Detect button, the results are displayed (see Figure~\ref{fig:website-bottom} for an illustrative example). The results panel presents the conclusion produced by our detector:
\begin{itemize}[leftmargin=*]
    \item \textbf{Conclusion box} (blue): It displays the conclusion of our statistical hypothesis test, together with the rationale for this conclusion, by reporting the corresponding $p$-value computed from the empirical null distribution and the user-specified significance level $\alpha$.
    \item \textbf{Interpretation and Suggestions} (blue): Explains how to interpret the $p$-value and the meaning of the significance level $\alpha$ for users outside the statistics community. It also provides practical guidance on how to use the model.
    \item \textbf{User feedback panel} (white):  Allows users to indicate whether the detection result aligns with their expectations. This helps us collect feedback data for future improvement.
\end{itemize}
As a concrete example, we paste the abstract of \citet{li2021adversarially} into the input box in the upper panel of Figure~\ref{fig:website}. This paper was published in 2021, prior to the release of ChatGPT, and is therefore likely written by humans. Our detector reports a $p$-value of 0.412, and thus fails to reject the null hypothesis that the text is human-authored. 

Next, we prompt GPT-5 to rewrite the same abstract (see the prompt and generated text in Section~\ref{sec:website-experiments} of the Supplementary Material) and paste the rewritten version into the input box. As shown in the lower panel of Figure~\ref{fig:website}, the resulting $p$-value is very small (0.0000), leading to the rejection of the null hypothesis. This case study illustrates the effectiveness of our detector in distinguishing human-written text from LLM-generated text.

\begin{figure}[H]
    \centering
    \begin{minipage}{0.48\linewidth}
        \centering
        \includegraphics[width=\linewidth]{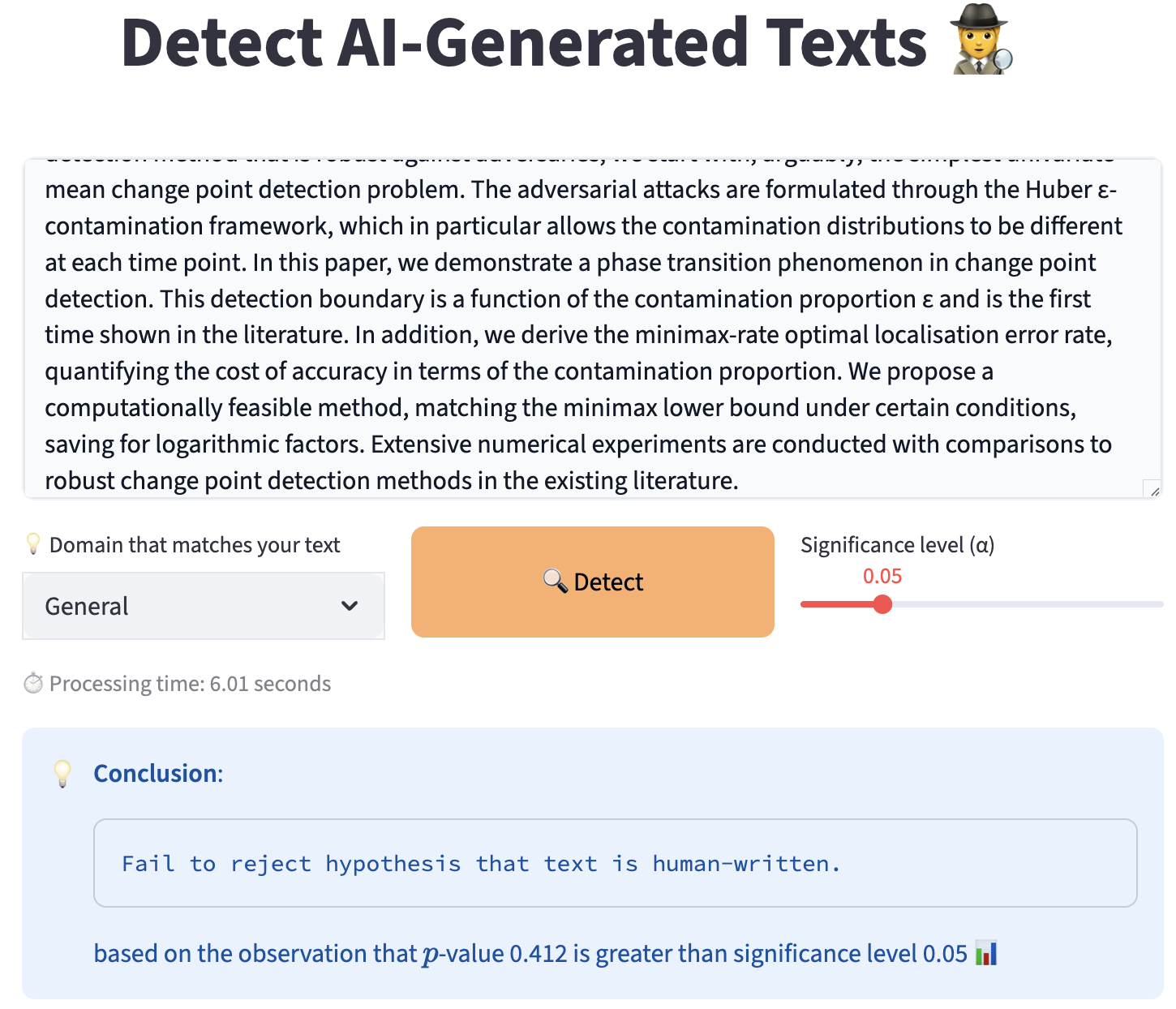}
        \caption*{}
    \end{minipage}
    \hfill
    \begin{minipage}{0.48\linewidth}
        \centering
        \includegraphics[width=\linewidth]{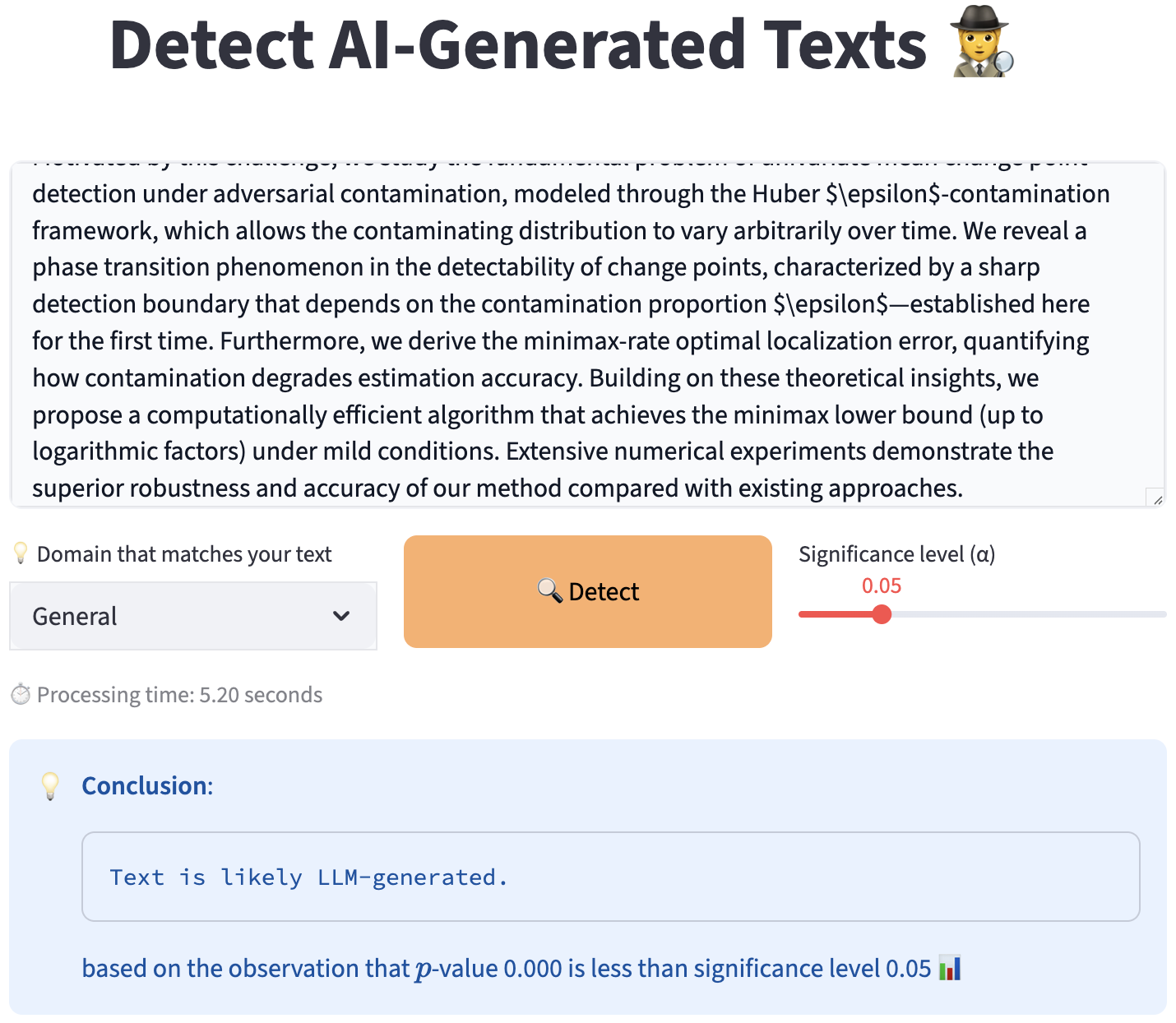}
        \caption*{}
    \end{minipage}
    \vspace{-50pt}
    \caption{Case study: detecting whether the abstract of the paper \citet{li2021adversarially} (left) and its rewritten version (right) are LLM-generated or not.}
    \label{fig:website-example}
    \vspace{-24pt}
\end{figure}

\section{Discussion}\label{sec:conclusion}
We develop a publicly accessible LLM detection tool in this paper. Compared to existing detectors in the literature, ours does not rely on watermarks or knowledge of the specific LLM used to generate the text. Furthermore, it demonstrates superior empirical performance in distinguishing between human and LLM-authored text while maintaining control over the type-I error. Given the rapid evolution of generative AI capabilities, ranging from texts to images and videos, extending our proposal to detect such AI-generated content remains a vital direction for future research.

\bibliography{bibliography.bib}



\phantomsection\label{supplementary-material}
\bigskip
\appendix

\renewcommand{\thesection}{\Alph{section}}

\section{Proofs}\label{sec:proof}
\subsection{Proof of Theorem \ref{thm:statdiff}}
    For any fixed $t$, Recall that $q_t^\mathcal{M}$ is defined as
    \begin{equation*}
        q_t^\mathcal{M}(x|\bm{x}_{<t}) = \frac{\exp\left(\frac{1}{\tau}\ell_t^\mathcal{M}(x|\bm{x}_{<t})\right)}{\sum_{x\in\mathcal{V}}\exp\left(\frac{1}{\tau}\ell_t^\mathcal{M}(x|\bm{x}_{<t})\right)}.
    \end{equation*}
    Suppose $\mathcal{A}_t(\bm{x}_{<t}) = \arg\max\limits_{v\in\mathcal{V}} \ell_t^\mathcal{M}(v|\bm{x}_{<t})$ be the set of tokens that maximizes $\ell_t^\mathcal{M}(\bullet|\bm{x}_{<t})$.
    It follows that as $\tau \to 0$, for any $x \in\mathcal{V}$,  $q_t^\mathcal{M}(x|\bm{x}_{<t}) \to \frac{1}{|\mathcal{A}_t|}\mathbb{I}\{x \in \mathcal{A}_t(\bm{x}_{<t})\}$.
    When $\mathcal{S}=\mathcal{M}$, it follows that for all $t\geq 1$ and for any distribution $\mathbb{P}$,
\begin{equation}
     \mathbb{E}_{\bm{X}\sim \mathbb{P}}\ell_t^{\mathcal{M}}(X_t| \bm{X}_{<t}) \leq \mathbb{E}_{\bm{X}\sim \mathbb{P}}\{ \max_{v\in\mathcal{V}} \ell_t^{\mathcal{M}}(v| \bm{X}_{<t}) \}  =  \mathbb{E}_{\substack{\bm{X}_{<t}\sim \mathbb{P}\\\widetilde{X}_t \sim q_t^{\mathcal{M}}(\bullet | \bm{X}_{<t})}} [\ell_t^{\mathcal{M}}(\widetilde{X}_t | \bm{X}_{<t})].
\end{equation}
Therefore,
\begin{equation}
    \sum_t \mathbb{E}_{\bm{X}\sim \mathbb{P}}\ell_t^{\mathcal{M}}(X_t| \bm{X}_{<t})-\sum_t \mathbb{E}_{\substack{\bm{X}_{<t}\sim \mathbb{P}\\\widetilde{X}_t \sim q_t^{\mathcal{M}}(\bullet | \bm{X}_{<t})}} [\ell_t^{\mathcal{M}}(\widetilde{X}_t | \bm{X}_{<t})]\le 0.
\end{equation}
This indicates that $\mathbb{E}_{\bm{X} \sim \mathbb{P}} S_{\textrm{Fast}}(\bm{X})\leq 0$. On the other hand, as temperature $\tau \to 0$, if $\mathbf{X}\sim \mathbb{Q}^\mathcal{M}$,
\begin{equation*}
    \ell_t^{\mathcal{M}}(X_t| \bm{X}_{<t}) - \mathbb{E}_{\substack{\bm{X}_{<t}\sim \mathbb{Q}^\mathcal{M}\\\widetilde{X}_t \sim q_t^{\mathcal{M}}(\bullet | \bm{X}_{<t})}} [\ell_t^{\mathcal{M}}(\widetilde{X}_t | \bm{X}_{<t})]= 0
\end{equation*}
almost surely, which indicates $\mathbb{E}_{\bm{X} \sim \mathbb{P}} S_{\textrm{Fast}}(\bm{X}) =0$. This finishes the proof.

\subsection{Proof of Theorem \ref{thmLw}}

We first introduce the technical conditions needed for proving Theorem \ref{thmLw}.
\begin{assumption}[Equal variance]\label{ass:equal-variance}
    For any non-constant witness function $w$, define 
    \begin{align*}
        \sigma_{q,L}^2 \coloneqq \frac{1}{L} \sum_{t=1}^L \textup{Var}_{\widetilde{X}_t\sim q_t}\left( w_t(\widetilde{X}_t|\widetilde{X}_{<t}) \right), \quad
        \sigma_{p,L}^2 \coloneqq \frac{1}{L} \sum_{t=1}^L \textup{Var}_{\widetilde{X}_t\sim p_t}\left( w_t(\widetilde{X}_t|\widetilde{X}_{<t}) \right).
    \end{align*}
    $\sigma_{q,L}^2, \sigma_{p,L}^2 $ are lower bounded by some constant $\sigma_w^2 >0$ almost surely. Moreover, $\sigma_{q,L} / \sigma_{p,L} \to 1 $ in probability as $L \to \infty$.
\end{assumption}

\begin{assumption}\label{ass:variance-ratio}
    For any witness function $w$, define
    \begin{align*}
        \bar\sigma_{q,L}^2 = \frac{1}{L} \sum_{t=1}^L \textup{Var}_{\bm{X}\sim q}\left( w_t(\widetilde{X}_t|\widetilde{X}_{<t}) \right), \quad
        \bar\sigma_{p,L}^2 = \frac{1}{L} \sum_{t=1}^L \textup{Var}_{\bm{X}\sim p}\left( w_t(\widetilde{X}_t|\widetilde{X}_{<t}) \right) .
    \end{align*}
    If $\bm{X}\sim q$, then $\bar\sigma_{q,L}^2/\sigma_{q,L}^2 \to 1$ in probability. If $\bm{X}\sim p$, then $\bar\sigma_{p,L}^2/\sigma_{p,L}^2 \to 1$ in probability.
\end{assumption}

\begin{lemma}\label{thm: convergence-rate-MCLT}
    Let $\bm{X} = (X_1,\ldots X_n)$ be sequences of real valued random variables satisfying for all $1\leq t\leq n$, 
    \begin{equation*}
        \mathbb{E} (X_t | X_{<t}) = 0\quad \textup{almost surely.}
    \end{equation*}
    Let $\sigma_t^2 = \mathbb{E} (X_t^2\big| X_{<t})$, $\bar\sigma_t^2 = \mathbb{E}(X_t^2)$, $s_n^2 = \sum_{t=1}^n \bar\sigma_t^2$ and $V_n^2 = \sum_{t=1}^n \sigma_t^2/s_n^2$. Suppose $|X_n|$ is bounded by some constant almost surely for all $n$ and $s_n/\sqrt{n}$ is bounded away from zero and $V_n^2\to 1$ in $L^1$. Then 
    \begin{equation*}
        \sup_{z\in \mathbb{R}} \left| \mathbb{P}\left( \frac{\sum_{t=1}^n X_t}{\sqrt{\sum_{t=1}^n \sigma_t^2}} \leq z \right) - \Phi(z) \right| \to 0,
    \end{equation*}
    where $\Phi(\bullet)$ is the cumulative distribution function of standard normal distribution.
\end{lemma}
\begin{proof}
The conclusion directly follows from martingale central limit theorem, see e.g. Corollary 1 of \cite{Bolthausen_1982_exact}.
\end{proof}

\begin{lemma}\label{lem:Taylorexpansion-of-normal-cdf}
    Suppose $X$ is a random variable. Let $\Phi$ and $\phi$ be the cumulative distribution function and probability density function of standard normal distribution. Then for any random variable $X$, 
    \begin{eqnarray}
        \mathbb{E}\Phi(z_\alpha +X) \geq \min\{ 1-\alpha, \alpha + \Phi'(z_\alpha)\mathbb{E}X\}, \nonumber
    \end{eqnarray}
    where $0<\alpha<1/2$, $z_\alpha$ is the $\alpha$-th quantile of standard normal distribution.
\end{lemma}
\begin{proof}
    The proof directly follows from Lemma S2 in \cite{zhou2025adadetectgpt}.
\end{proof} 

Now, we proceed to prove Theorem~\ref{thmLw}. Noted that our test statistics can be decomposed as $S(\bm{X})=S^{(1)}(\bm{X})-S^{(2)}(\bm{X})$ with $S^{(1)}(\bm{X}),S^{(2)}(\bm{X})$ defined by
\begin{eqnarray}\label{eq:Tw-decomposition}
   S^{(1)}(\bm{X}) &=&\frac{\sum_t [w_t(X_t|X_{<t})-\mathbb{E}_{\widetilde{X}_t \sim p_t} w_t(\widetilde{X}_t|X_{<t})]}{\sqrt{\sum_t \textrm{Var}_{\widetilde{X}_t \sim q_t} (w_t(\widetilde{X}_t|X_{<t}))}}\nonumber\\
    S^{(2)}(\bm{X})&=&\frac{\sum_t [\mathbb{E}_{\widetilde{X}_t \sim q_t} w_t(\widetilde{X}_t|X_{<t})-\mathbb{E}_{\widetilde{X}_t \sim p_t} w_t(\widetilde{X}_t|X_{<t})]}{\sqrt{\sum_t \textrm{Var}_{\widetilde{X}_t \sim q_t} (w_t(\widetilde{X}_t|X_{<t}))}},
\end{eqnarray}
The TNR can be represented as
\begin{eqnarray}
    \mathbb{P}_{\bm{X}\sim p}\left(S(\bm{X}) \leq z_\alpha \right) &=& \mathbb{P}_{\bm{X}\sim p}\left( S^{(1)}(\bm{X}) \leq z_\alpha + S^{(2)}(\bm{X}) \right)
\end{eqnarray}
It is easy to verify that when $\bm{X} \sim p$, $S^{(1)}(\bm{X}) \sigma_{q,L}/\sigma_{p,L}$ converges to standard normal distribution. Specifically, using Lemma \ref{thm: convergence-rate-MCLT}, we obtain that 
\begin{eqnarray}
    \mathbb{P}_{\bm{X}\sim p}\left( S(\bm{X}) \leq z_\alpha \right)  &=& \mathbb{P}_{\bm{X}\sim p}\left( S^{(1)}(\bm{X})\frac{\sigma_{q,L}}{\sigma_{p,L}} \leq (z_\alpha + S^{(2)}
    (\bm{X}))\frac{\sigma_{q,L}}{\sigma_{p,L}} \right)\nonumber\\
    &\geq& \Phi(z_\alpha + S^{(2)}(\bm{X})) + \left(\Phi\left((z_\alpha + S^{(2)}
    (\bm{X}))\frac{\sigma_{q,L}}{\sigma_{p,L}}\right) - \Phi(z_\alpha + S^{(2)}(\bm{X})) \right)\nonumber\\
    &&\qquad +o_p(1)\nonumber\\
    &\geq& \Phi(z_\alpha + S^{(2)}(\bm{X})) - \sup_{z\in\mathbb{R}} |\phi(z)| \times \left| z_\alpha + S^{(2)}(\bm{X}) \right| \times \left| \frac{\sigma_{q,L}}{\sigma_{p,L}}-1 \right|+o_p(1) \nonumber
\end{eqnarray}
Under Assumption \ref{ass:equal-variance}, $\sigma_{q,L} / \sigma_{p,L} \to 1$ in probability, we obtain
\begin{equation*}
    \mathbb{P}_{\bm{X}\sim p}\left( S(\bm{X}) \leq z_\alpha \right) \geq \Phi(z_\alpha + S^{(2)}(\bm{X})) + o_p(1).
\end{equation*}
Moreover, the remainder term $o_p(1)$ is uniformly integrable since $\mathbb{P}_{\bm{X}\sim p}\left( S(\bm{X})\right) $ and $ \Phi(z_\alpha + S^{(2)}(\bm{X})) $ are all bounded above. Take expectation on both sides, we have by Assumption~\ref{ass:equal-variance} that
\begin{eqnarray}
    \mathbb{P}_{\bm{X}\sim p}\left( S(\bm{X}) \leq z_\alpha \right) \geq \mathbb{E}\Phi(z_\alpha + S^{(2)}(\bm{X})) + o(1).\nonumber
\end{eqnarray}
Next, define $\widetilde\sigma_{q,L}^2 = \mathbb{E}_{\bm{X}\sim p}\sigma_{q,L}^2$. It follows that $L_w= \mathbb{E}\left\{S^{(2)}(\bm{X})\frac{\sigma_{q,L}}{\widetilde{\sigma}_{q,L}}\right\}$. Under the equal variance assumption in Assumption \ref{ass:equal-variance}, we also have $ \sigma_{q,L} - \widetilde\sigma_{q,L} \to 0$ in probability. It follows that for any $\epsilon >0$,
\begin{eqnarray}
    &&\mathbb{E}\Phi(z_\alpha + S^{(2)}(\bm{X}))  
    \\
    &=&\mathbb{E}\Phi(z_\alpha + S^{(2)}(\bm{X}))\mathbb{I}\{|\sigma_{q,L} - \widetilde\sigma_{q,L}| \leq\epsilon\} \nonumber\\
    &&\qquad + \mathbb{E}\Phi(z_\alpha + S^{(2)}(\bm{X}))\mathbb{I}\{|\sigma_{q,L} - \widetilde\sigma_{q,L}| >\epsilon\} \nonumber\\
    &\geq& \mathbb{E}\Phi(z_\alpha + S^{(2)}(\bm{X}))\mathbb{I}\{|\sigma_{q,L} - \widetilde\sigma_{q,L}| \leq\epsilon\} \nonumber\\
    &\geq & \mathbb{E}\Phi\left(z_\alpha + S^{(2)}(\bm{X})\frac{\sigma_{q,L}}{\widetilde\sigma_{q,L} + \text{sgn}(S^{(2)}(\bm{X}))\epsilon}\right)\mathbb{I}\{|\sigma_{q,L} - \widetilde\sigma_{q,L}| \leq\epsilon\} \nonumber\\
    &\geq& \mathbb{E}\Phi\left(z_\alpha + S^{(2)}(\bm{X})\frac{\sigma_{q,L}}{\widetilde\sigma_{q,L} + \text{sgn}(S^{(2)}(\bm{X}))\epsilon}\right) \nonumber\\
    &&\qquad -  \mathbb{E}\Phi\left((z_\alpha + S^{(2)}(\bm{X}))\frac{\sigma_{q,L}}{\widetilde\sigma_{q,L} + \text{sgn}(S^{(2)}(\bm{X}))\epsilon}\right)\mathbb{I}\{|\sigma_{q,L}- \widetilde\sigma_{q,L}| >\epsilon \} \nonumber\\
    &\geq& \mathbb{E}\Phi\left((z_\alpha +S^{(2)}(\bm{X}))\frac{\sigma_{q,L}}{\widetilde\sigma_{q,L} +\text{sgn}(S^{(2)}(\bm{X})) \epsilon}\right) - \mathbb{P}(|\sigma_{q,L} - \widetilde\sigma_{q,L}| >\epsilon), \nonumber
\end{eqnarray}
where the first inequality is obtained due to $\Phi$ is non-negative and the second inequality holds due to the monotonicity and boundedness of $\Phi$.  Together with Lemma \ref{lem:Taylorexpansion-of-normal-cdf} and Assumption \ref{ass:equal-variance}, we obtain
\begin{equation}\label{eqn:thm1-final-ineq}
\begin{split}
    \mathbb{P}_{\bm{X}\sim p}\left( S(\bm{X}) \leq z_\alpha \right)
    \geq&\min\left\{1-\alpha, \alpha + \phi(z_\alpha)\mathbb{E}\left\{S^{(2)}(\bm{X})\frac{\sigma_{q,L}}{\widetilde{\sigma}_{q,L}}\right\} \right\} \frac{\widetilde{\sigma}_{q,L}}{\widetilde{\sigma}_{q,L} + \text{sgn}(S^{(2)}(\bm{X}))\epsilon} \\
    &-\mathbb{P}\{|\sigma_{q,L} - \widetilde\sigma_{q,L}| \geq\epsilon\}  + o(1).
\end{split}
\end{equation}
Let $L\to \infty$ and using the fact that $\mathbb{E}\left\{S^{(2)}(\bm{X})\frac{\sigma_{q,L}}{\widetilde{\sigma}_{q,L}}\right\}=  L_w$, we obtain that TNR is asymptotically lower bounded by $\min\{1-\alpha, \alpha + \phi(z_\alpha)L_w) \}\frac{\widetilde{\sigma}_{q,L}}{\widetilde{\sigma}_{q,L} + \text{sgn}(S^{(2)})\epsilon}$. By taking $\epsilon \to 0$, then the conclusion of Theorem \ref{thmLw} follows.

\subsection{Proof of Theorem \ref{thm:typeI}}
By Glivenko-Cantelli Theorem,
\begin{equation*}
    \sup_{s\in\mathbb{R}}\left|\frac{1+\sum_{j=1}^m\mathbb{I}\{s < S(\bm{X}_j)\}}{m+1} - (1 - F(s))\right| \to 0
\end{equation*}
in probability, where $F$ is the cumulative distribution function of $S(\bm{X}_j)$. Noted that under null hypothesis, $S(\bm{X})$ follows distribution $F$. Therefore, $F(S(\bm{X}))$ exactly follows uniform distribution $U(0,1)$. Consequently, for any $\alpha\in(0,1)$
\begin{eqnarray}
    \mathbb{P}_{\bm{X}\sim \mathbb{P}}\left(\text{p-value} \leq \alpha\right)= \mathbb{P}_{X\sim \mathbb{P}}\left( 1-F(S(\bm{X})) \leq \alpha\right)+ o_p(1) \to \alpha
\end{eqnarray}
as $m\to \infty$. This finishes the proof of Theorem \ref{thm:typeI}.

\section{Data}

\subsection{Human Text: Source and Details on Processing}\label{sec:data-human}

\textbf{Human text source: details}. we describe the data resource used in this paper:
\begin{itemize}[leftmargin=*]
    \item Academia: the paper in PubMed\footnote{\url{https://github.com/armancohan/long-summarization?tab=readme-ov-file}} and the abstract in arXiv\footnote{\url{https://www.kaggle.com/datasets/Cornell-University/arxiv}}.
    \item Finance: the financial conversations\footnote{\url{https://huggingface.co/datasets/ceadar-ie/FinTalk-19k}}, opinion-based answering on the question in finance\footnote{\url{https://huggingface.co/datasets/LLukas22/fiqa}}, and the financial news\footnote{\url{https://huggingface.co/datasets/danidanou/Bloomberg_Financial_News}}
    \item Knowledge: the cleaned Wikipedia articles before March 2022\footnote{\url{https://huggingface.co/datasets/legacy-datasets/wikipedia}}
    \item Government: the U.S. government report dataset\footnote{\url{https://huggingface.co/datasets/launch/gov_report}} and its corresponding summarization written by human expert\footnote{\url{https://huggingface.co/datasets/ccdv/govreport-summarization}}
    \item Legislation: case holdings on legal decisions on the US court cases\footnote{\url{https://huggingface.co/datasets/coastalcph/lex_glue}} and the law of European Union written in English\footnote{\url{https://huggingface.co/datasets/coastalcph/multi_eurlex}}
    \item Medicine: description on patients\footnote{\url{https://huggingface.co/datasets/ncbi/Open-Patients}} and expert's explanation on the answer about the benchmarked medical questions\footnote{\url{https://huggingface.co/datasets/openlifescienceai/medmcqa}}.
    \item News: the news article from BBC\footnote{\url{https://huggingface.co/datasets/SetFit/bbc-news}}\footnote{\url{https://huggingface.co/datasets/EdinburghNLP/xsum}} and CNN\footnote{\url{https://huggingface.co/datasets/AyoubChLin/CNN_News_Articles_2011-2022}}.
    \item Users review: food review on Amazon\footnote{\url{https://www.kaggle.com/datasets/snap/amazon-fine-food-reviews}}, product review on Yelp\footnote{\url{https://huggingface.co/datasets/Yelp/yelp_review_full/}}, and movie review on ImDB\footnote{\url{https://huggingface.co/datasets/stanfordnlp/imdb}}. 
\end{itemize}

\textbf{Details of Text Pre-processing}. To ensure the high quality of the collected texts, we discarded those containing no more than 20 words, as both theoretical and empirical findings \citep{bao2024fastdetectgpt, zhou2025adadetectgpt} suggest that very short texts provide limited information for determining whether they are LLM-generated. We also avoided overly long texts, as they impose significant computational burdens during fine-tuning due to the quadratic runtime of the classical Transformer architecture. To handle extremely long texts (e.g., certain entries from Wikipedia), we randomly selected 7–8 consecutive sentences from the original content instead. In addition, we eliminated texts containing a substantial amount of repetition. Specifically, we computed the 3-gram statistics for each text, which ranges from 0 to 1, where larger values indicate more repetition. Following the comments in \citet{huang2025attributing}, we set the threshold for the 3-gram statistics to 0.4.

\subsection{Machine data generation}\label{sec:data-machine}

The procedure for generating LLM-generated texts in each category is as follows. First, for each text category, we randomly sampled 225 texts. Then, for each selected text, we randomly chose one LLM model from \texttt{grok-3-mini}, \texttt{gemini-2.5-flash}, and \texttt{gpt-4o} to generate a corresponding text. In addition, we randomly selected one prompt from a pool of more than 20 candidate prompts, designed to capture diverse linguistic styles of different models. Typical examples include:

\begin{tcolorbox}[colback=white,colframe=black!50,title=Prompt for generating LLM texts]
Expand but not extend the paragraph in a persuasive style.
\end{tcolorbox}

We also added necessary system prompts to ensure that the LLMs do not produce unnecessary text. The specific system prompts for the rewrite, polish, and expand tasks are listed below.

\begin{tcolorbox}[colback=white,colframe=black!50,title=System prompt for the rewriting task]
You are a professional rewriting expert and you can help paraphrasing this paragraph in English without missing the original details. Please keep the length of the rewritten text similar to the original text. Return ONLY the rewritten version. Do not explain changes, do not give multiple options, and do not add commentary.
\end{tcolorbox}

\begin{tcolorbox}[colback=white,colframe=black!50,title=System prompt for the polish task]
    You are a professional polishing expert and you can help polishing this paragraph. Return ONLY the polished version. Do not explain changes, do not give multiple options, and do not add commentary.
\end{tcolorbox}

\begin{tcolorbox}[colback=white,colframe=black!50,title=System prompt for the polish task]
    You are a professional writing expert and you can help expanding this paragraph. Return ONLY the expanded version. Do not explain, do not give multiple options, and do not add commentary.
\end{tcolorbox}

When calling the APIs of these LLMs, for simplicity, we did not set the temperature or perform top-$k$ or nucleus sampling. 

\section{Experiments: Details}
\subsection{Implementation Details}\label{sec:method-details}
\textbf{Fine-tuning setting of our method.} In our implementation, we initialize the $w$ function using \texttt{google/gemma-3-1b-pt}\footnote{\url{https://huggingface.co/google/gemma-3-1b-pt}}. The model is then fine-tuned with LoRA \citep{hu2022lora}, implemented via the \texttt{peft} library, where the rank parameter is set to 4, \texttt{lora\_alpha} to 16, and \texttt{lora\_dropout} to 0.05. All other parameters are kept at their default settings.

\textbf{Estimating the distribution of statistics under $\mathcal{H}_0$.}
For evaluation on external datasets, we estimate the distribution of the test statistic under $\mathcal{H}_0$ using human-written texts that were not used in fine-tuning. The corresponding histograms are presented in Figure~\ref{fig:null-distr}. From this figure, we observe that the empirical null distributions deviate from normality in most domains. This suggests that deriving a closed-form asymptotic distribution under $\mathcal{H}_0$ seems to be difficult, providing empirical justification for using the empirical distribution to estimate $p$-values.

\begin{figure}[!t]
    \centering
    \includegraphics[width=1.0\linewidth]{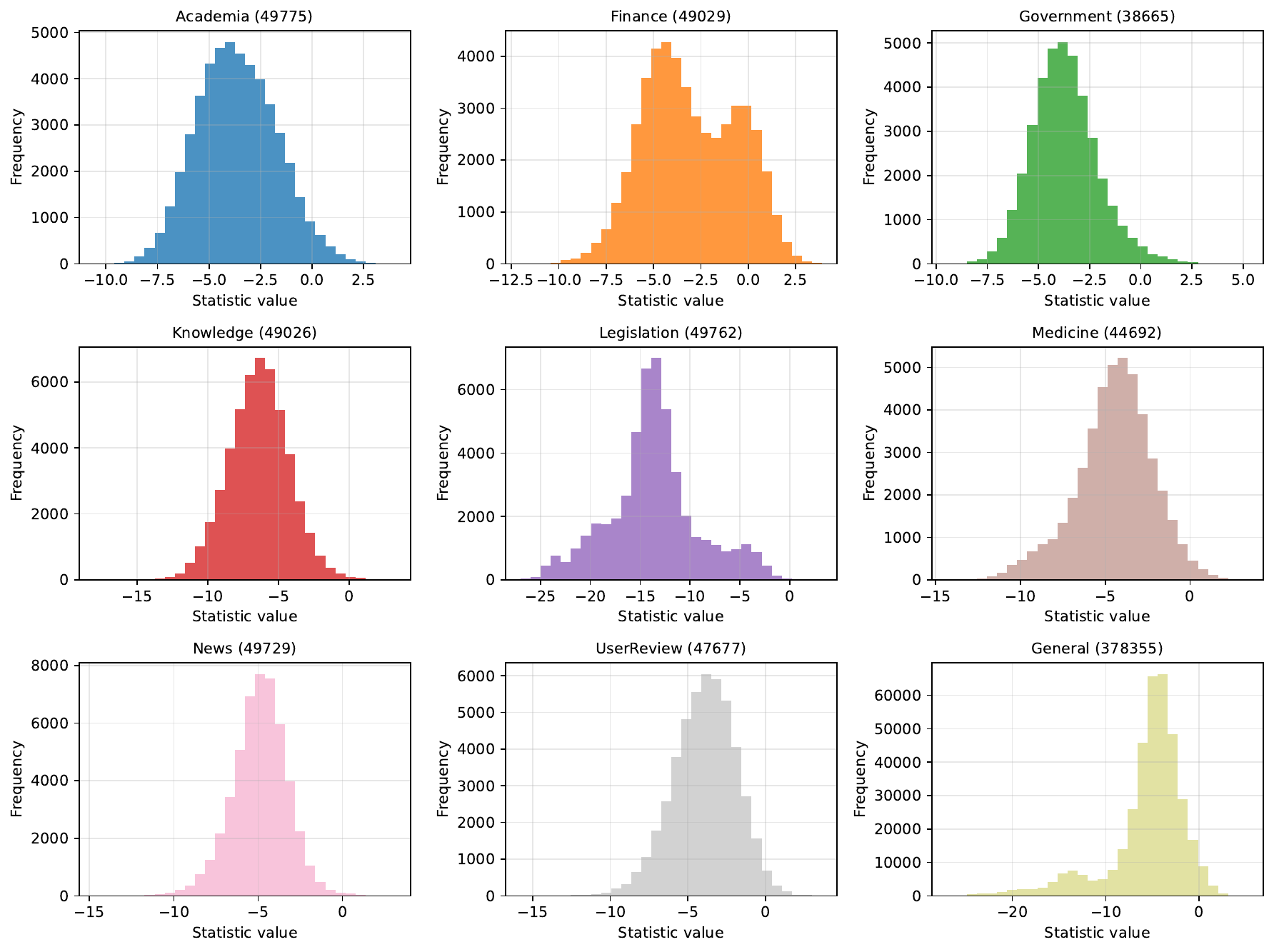}
    \caption{Empirical distributions of the test statistics for human-written texts.}
    \label{fig:null-distr}
\end{figure}

\textbf{Implementation of baselines}. For fairness, all training-free baselines use the same sampling model as our method, \texttt{google/gemma-3-1b-pt}. For methods that require a scoring model, we use \texttt{google/gemma-3-1b-it}\footnote{\url{https://huggingface.co/google/gemma-3-1b-it}}, an instruction-tuned version of \texttt{google/gemma-3-1b-pt}. For the training-based method ImBD, we use the same sampling model as ours and adopt the default LoRA settings provided in their paper. For RoBERTa Detector and RADAR, we directly use the checkpoints provided on Hugging Face, i.e., \texttt{openai-community/roberta-large-openai-detector}\footnote{\url{https://huggingface.co/openai-community/roberta-large-openai-detector}} and \texttt{TrustSafeAI/RADAR-Vicuna-7B}\footnote{\url{https://huggingface.co/TrustSafeAI/RADAR-Vicuna-7B}}.

\textbf{Computational device}. The fine-tuning procedures and experiments are conducted on an Ubuntu 22.04 platform equipped with a 16-vCPU AMD EPYC 9K84 96-Core Processor and 150 GB RAM. The system is attached to an NVIDIA H20-NVLink GPU with 96 GB VRAM, and all experiments are run using Python 3.10.

\subsection{Real-world application at Section~\ref{sec:website}}\label{sec:website-experiments}

\begin{tcolorbox}[colback=white,colframe=black!50,title=Prompt to GPT5]
    Please rewrite the abstract of one paper: ``Change point detection is becoming increasingly popular in many application areas. On one hand, most of the theoretically-justified methods are investigated in an ideal setting without model violations, or merely robust against identical heavy-tailed noise distribution across time and/or against isolate outliers; on the other hand, we are aware that there have been exponentially growing attacks from adversaries, who may pose systematic contamination on data to purposely create spurious change points or disguise true change points. In light of the timely need for a change point detection method that is robust against adversaries, we start with, arguably, the simplest univariate mean change point detection problem. The adversarial attacks are formulated through the Huber $\epsilon$-contamination framework, which in particular allows the contamination distributions to be different at each time point. In this paper, we demonstrate a phase transition phenomenon in change point detection. This detection boundary is a function of the contamination proportion $\epsilon$ and is the first time shown in the literature. In addition, we derive the minimax-rate optimal localisation error rate, quantifying the cost of accuracy in terms of the contamination proportion. We propose a computationally feasible method, matching the minimax lower bound under certain conditions, saving for logarithmic factors. Extensive numerical experiments are conducted with comparisons to robust change point detection methods in the existing literature.''
\end{tcolorbox}

\begin{tcolorbox}[colback=white,colframe=black!50,title=Text outputed by GPT5]
Here’s a polished and more natural rewrite of your abstract — preserving all the technical content but improving flow, clarity, and academic tone:

Abstract (Rewritten):
Change point detection has attracted growing attention across a wide range of applications. While many existing methods enjoy strong theoretical guarantees, they are typically developed under idealized assumptions—such as model correctness or homogeneous noise—and at best demonstrate robustness to isolated outliers or heavy-tailed but identical noise distributions. In contrast, modern data environments increasingly face adversarial contamination, where systematic perturbations are deliberately introduced to obscure true change points or to fabricate spurious ones.

Motivated by this challenge, we study the fundamental problem of univariate mean change point detection under adversarial contamination, modeled through the Huber $\epsilon$-contamination framework, which allows the contaminating distribution to vary arbitrarily over time. We reveal a phase transition phenomenon in the detectability of change points, characterized by a sharp detection boundary that depends on the contamination proportion $\epsilon$—established here for the first time. Furthermore, we derive the minimax-rate optimal localization error, quantifying how contamination degrades estimation accuracy. Building on these theoretical insights, we propose a computationally efficient algorithm that achieves the minimax lower bound (up to logarithmic factors) under mild conditions. Extensive numerical experiments demonstrate the superior robustness and accuracy of our method compared with existing approaches.

Would you like me to make it slightly shorter and punchier (for a NeurIPS/ICML-style abstract), or keep it formal and complete (for a statistics journal like Annals of Statistics or JASA)?
\end{tcolorbox}

\end{document}